\documentclass{article}

\usepackage{times}
\usepackage{graphicx} % more modern
\usepackage{natbib}

\usepackage{algorithm}
\usepackage{algorithmic}
\usepackage{amsmath,amsthm}
\usepackage{amsfonts}
\usepackage{epstopdf}
\DeclareGraphicsExtensions{.pdf,.eps}
\usepackage{enumitem}
\usepackage{mysymbols}
\usepackage{tikz}
\usetikzlibrary{positioning,arrows,shapes}
\newcommand{\axes}{
\draw (-1,0) -- (1,0);
\draw (0, -1) -- (0, 1);
\draw [fill] (0,0) circle   (1.5pt);
}

\tikzstyle{controller}=[circle,draw=blue!50,fill=blue!30,thick,
inner sep=0pt,minimum size=20mm]

\tikzstyle{state} = [circle,draw=black!100,fill=blue!20,thick,
inner sep=0pt,minimum size=20mm
]

\tikzstyle{observ}=[circle,draw=black!100,fill=black!10,thick,
inner sep=0pt,minimum size=20mm]

\tikzstyle{pre}=[<-,shorten <=1pt,>=stealth',thick]

\tikzstyle{post}=[->,shorten >=1pt,>=stealth',thick]

\tikzstyle{block} = [rectangle, draw, fill=blue!0
,text width=25em
, text centered, rounded corners
, minimum height=2em
]

\tikzstyle{blockred} = [rectangle, draw, fill=blue!30
,text width=25em
, text centered, rounded corners
, minimum height=2em
]

\tikzstyle{blockS} = [rectangle, 
,text width=8em
, text centered, rounded corners
, minimum height=2em
]

\tikzstyle{block2} = [rectangle, draw, fill=blue!0
,text width=8em
, text centered, rounded corners
, minimum height=2em
]
\tikzstyle{block2red} = [rectangle, draw, fill=red!0
,text width=8em
, text centered, rounded corners
, minimum height=2em
]
\tikzstyle{line} = [draw, -latex']

\tikzstyle{cloud} = [draw, ellipse
]

\tikzstyle{blockblue} = [rectangle, draw, fill=red!0
,text width=25em
, text centered, rounded corners
, minimum height=3em
]
\tikzstyle{blockred} = [rectangle, draw, fill=red!0
,text width=25em
, text centered, rounded corners
, minimum height=3em
]
\tikzstyle{blockyel} = [rectangle, draw, fill=blue!15
,text width=15em
, text centered, rounded corners
, minimum height=3em
]
\tikzstyle{line} = [draw, -latex']
\tikzstyle{cloud} = [draw, ellipse,fill=orange!0%, node distance=1cm
]

\tikzstyle{blockeqn} = [rectangle, draw, fill=blue!15
,minimum width=5em
, text centered, rounded corners
, minimum height=1em
]

\newcommand{\bm}{\boldsymbol}

\renewcommand{\P}{\mathcal{P}}
\newcommand{\PH}{\P_{\hspace{-1pt}H}}
\newcommand{\PHt}{\P_{\hspace{-1pt}H'}}
\newcommand{\X}{\mathcal{X}}
\newcommand{\Y}{\mathcal{Y}}

\renewcommand{\H}{\mathcal{H}}
\newcommand{\F}{\mathcal{F}}
\newcommand{\xmark}{\textit{\sffamily X}}

\newcommand{\prm}{\theta} % denotes the parameter (theta)
\newcommand{\prmv}{\bm\theta} % denotes the parameter (theta)
\newcommand{\mix}{\psi} % denotes the mixing coefficient (psi)
\newcommand{\prmset}{\Theta}
\newcommand{\ac}{\hat i}

\newcommand{\sign}{\operatorname{sign}}

\usepackage{bbm}
\newcommand{\E}{\mexp}

\newcommand{\obsOpr}{T_{\bm\prm}} 

\newcommand{\ot}{\bar{n}}

\newcommand{\vo}{\bm c}
\newcommand{\uo}{\bm b}

\renewcommand{\O}{O_{P}}

\newcommand{\prmotv}{ \bar{\bm{\prm}}}
\newcommand{\prmot}{\bar{\prm}}

\newcommand{\sm}{\!-\!}

\renewcommand{\u}{\bm u}
\renewcommand{\v}{\bm v}
\newcommand{\At}{\bar A}
\newcommand{\Ht}{\bar H}

\newcommand{\alpmin}{ \a^{\min}}
\newcommand{\alpmax}{ \a^{\max}}

\newcommand{\av}{\bar{\bm a}}

\renewcommand{\t}{\tau}
\newcommand{\dt}{\Delta \t}

\newcommand{\ths}{h_T}

\newcommand{\gnn}{f}
\newcommand{\gcc}{g}

\newcommand{\feasreg}{\Gamma}

\newcommand{\supp}{\operatorname{supp}}
\newcommand{\row}{_{{\textrm{\tiny \textup{in}}}}}

\newcommand{\tran}{{\textrm{\tiny \textup{\sffamily T}}}}
\renewcommand{\a}{\alpha}
\renewcommand{\b}{\beta}
\newcommand{\g}{\gamma}

\renewcommand{\t}{\tau}
\newcommand{\frob}{\textrm{{\tiny \textup{F}}}}

\newcommand{\kernel}{{\bar{K}}}
\newcommand{\kernelExp}{\mathcal{K}}

\newcommand{\mom}{{M}}
\newcommand{\momK}{{\mathcal{M}}}
\newcommand{\momT}{{G}}
\newcommand{\emomT}{{\hat \momT}}
\newcommand{\emom}{{\hat \mom}}
\newcommand{\emomK}{{\hat{\mathcal{M}}}}
\newcommand{\momTK}{{\mathcal{G}}}
\newcommand{\emomTK}{{\hat{\mathcal{G}}}}
\newcommand{\dmom}[1]{R^{(#1)}}
\newcommand{\dmomT}[1]{F^{(#1)}}
\newcommand{\dmomo}{\kappa}

\newcommand{\edmom}[1]{ \hat R^{(#1)}}
\newcommand{\edmomT}[1]{\hat F^{(#1)}}
\newcommand{\edmomo}{\hat{\kappa}}

\newcommand{\init}{\pi^0}
\newcommand{\initv}{\bm \pi^0}
\newcommand{\initot}{\bar{\bm\pi}^0}

\newcommand{\statv}{\bm\pi}
\newcommand{\stat}{\pi}
\newcommand{\statot}{\bar{\stat}}
\newcommand{\statotv}{\bar{\statv}}
\newcommand{\dav}{{{\bm \delta}^{\operatorname{out}}}}
\newcommand{\dalpv}{{{\bm \delta}^{\operatorname{in}}}}
\newcommand{\da}{{\delta}^{\operatorname{out}}}
\newcommand{\dalp}{{\delta}^{\operatorname{in}}}

\newcommand{\eAt}{\hat\At}

\newtheorem{theorem}{Theorem}

\newtheorem{lemma}{Lemma}

\title{Learning Parametric-Output HMMs with Two Aliased States}
\author{Roi Weiss\\Ben-Gurion University \and Boaz Nadler \\ Weizmann Institute of Science}
\begin{document} 
\maketitle
%\twocolumn[
%\icmltitle{Learning Parametric-Output HMMs with Two Aliased States}

% unless you've provided the [accepted] option to the icml2015
%\icmlauthor{Roi Weiss}{roiwei@cs.bgu.ac.il}
%\icmladdress{Your Fantastic Institute,
%            314159 Pi St., Palo Alto, CA 94306 USA}
%\icmlauthor{Boaz Nadler}{email@coauthordomain.edu}
%\icmladdress{Their Fantastic Institute,
%            27182 Exp St., Toronto, ON M6H 2T1 CANADA}

%\icmlkeywords{boring formatting information, machine learning, ICML}

%\vskip 0.3in
%]

%%%%%%%%%%%%%%%%%%%%%%%%%%%%%%%%%%%%%%%%%%%%%%%%%%%%%%%%5
\begin{abstract} 
In various applications involving hidden Markov models (HMMs), some of the hidden states are \textit{aliased}, having identical output distributions.
The
minimality, identifiability and learnability of such aliased HMMs have been long standing  problems, with only partial solutions provided thus far. In this paper we focus on  parametric-output HMMs, whose output distributions come from a parametric family, and that have exactly two aliased states. For this class, we present a complete characterization of their minimality and identifiability. Furthermore, for a large family of parametric output distributions, we derive computationally efficient and statistically consistent algorithms to detect the presence of aliasing and learn the aliased HMM transition and emission parameters. We illustrate our theoretical analysis by  several simulations.   
\end{abstract} 
%%%%%%%%%%%%%%%%%%%%%%%%%%%%%%%%%%%%%%%%%%%%%%%%%%%%%%%%%

%%%%%%%%%%%%%%%%%%%%%%%%%%%%%%%%%%%%%%%%%%%%%%%%%%%%%%%%%%
\section{Introduction}

HMMs are a fundamental tool in the analysis of time series. A\ discrete time HMM\ with \(n\) hidden states is characterized by a $n\times n$ 
transition matrix, and by the emissions probabilities from these $n$ states. In several applications, the HMMs, or more general processes such as partially observable Markov decision processes, are {\em aliased}, with some states having identical output distributions.
In modeling of ion channel gating, 
for example, one postulates that at any given time an ion channel can be in only one of a 
finite number of hidden states,
some of which are open and conducting current while others are closed, see e.g. \citet{fredkin1992maximum}. 
Given electric current measurements,
one fits an aliased HMM and infers important biological insight regarding the gating process.
Other examples appear in the fields of reinforcement learning \citep{chrisman1992reinforcement,mccallum1995instance,shani2004resolving,shani2005model} and robot navigation \citep{jefferies2008robotics,zatuchna2009learning}.
In the latter case, aliasing occurs whenever different spatial locations appear (statistically) identical to the robot, given its limited sensing devices. As a last example, HMMs with several silent states that do not emit any output \citep{leggetter1994speaker,stanke2003gene,brejova2007most}, can also be viewed as aliased. 

Key notions related to the study of HMMs, be them aliased or not, are their minimality, identifiability and learnability: 
\begin{itemize}%[leftmargin=10pt,topsep = 1pt, itemsep=1pt]
\item[]{\em Minimality.} Is there an HMM 
with fewer states that induces the same distribution over all output sequences?
\item[]{\em Identifiability.} Does
the distribution over all output sequences
%all output sequence distributions 
%of a minimal HMM  
uniquely determines the HMM's parameters, up to a permutation of its hidden states?  
\item[]{\em Learning.} Given a long output sequence from a minimal and identifiable HMM, efficiently learn its parameters. 
\end{itemize}
%\noindent{\em Minimality.} Under which conditions
%there is no other HMM 
%with fewer number of states, but with the same distribution over all output sequences?
%
%\noindent{\em Identifiability.} Under which conditions
%all output-sequence distributions of a minimal HMM  uniquely determine its parameters, up to a permutation of its hidden states?  
%
%\noindent{\em Learning.} Given a long output sequence from a minimal and identifiable HMM, efficiently learn its parameters. 

For non-aliased HMMs, these notions have been intensively studied and by now are relatively well understood, see for example \citet{petrie1969probabilistic,finesso1990consistent,leroux1992maximum,allman2009identifiability} and \citet{Cappe_HMM_book}. 
The most common approach to learn the parameters of an HMM is via the Baum-Welch iterative algorithm \citep{baumwelch1970}.
Recently,
tensor decompositions and other  computationally efficient spectral methods have been developed to learn non-aliased HMMs 
\citep{DBLP:conf/colt/HsuKZ09,Siddiqi10,DBLP:conf/colt/Anand12,kontorovich2013learning}. 

In contrast, the minimality, identifiability and learnability of aliased HMMs have been long standing problems, with only partial solutions provided thus far. 
For example, \citet{blackwell1957identifiability} characterized the identifiability of a specific aliased HMM\ with 4 states. The identifiability of deterministic output HMMs, where each hidden state outputs a deterministic symbol, was partially resolved by \citet{ito1992identifiability}. To the best of our knowledge, precise characterizations of the minimality, identifiability and learnability of probabilistic output HMMs with aliased states are still open problems. In particular, the recently developed tensor and spectral methods mentioned above, explicitly require the HMM to be non-aliasing, and are not directly applicable to learning aliased HMMs.  
 
\paragraph{Main results.}In this paper we study the minimality, identifiability and learnability of parametric-output HMMs that have \textit{exactly two} aliased states. This is the simplest possible class of aliased HMMs, and as shown below, even its analysis is far from trivial. Our main contributions are as follows:\ First, we provide a complete characterization of their minimality and identifiability, deriving necessary and sufficient conditions for each of these notions to hold. Our identifiability conditions are easy to check for any given 2-aliased HMM, and extend those of \citet{ito1992identifiability} for the case of deterministic outputs. 
Second, we solve the problem of learning a possibly aliased HMM, from a long sequence of its outputs.  To this end, we first derive an algorithm to \textit{detect} whether an observed output sequence corresponds to a non-aliased HMM or to an aliased one. In the former case, the HMM\  can be learned by various methods, such as \citet{DBLP:conf/colt/Anand12,kontorovich2013learning}. In the latter case we show how the aliased states can be identified
and present a method to recover the HMM parameters. Our approach is applicable to any family of output distributions whose mixtures are efficiently learnable. Examples include  high dimensional Gaussians and products distributions, see \citet{feldman2008learning,belkin2010polynomial,DBLP:conf/colt/Anand12} and references therein.  After learning the output mixture parameters, our moment-based algorithm requires only a single pass over the data. It is possibly the first statistically consistent and computationally efficient scheme to handle 2-aliased HMMs.
While our approach may be extended to more complicated aliasing, such cases are beyond the scope of this paper.
We conclude with some simulations illustrating the performance of our suggested algorithms.

%%%%%%%%%%%%%%%%%%%%%%%%%%%%%%%%%%%%%%%%%%%%%%%%%%%%%%%%%%%
\section{Definitions \& Problem Setup}
\paragraph{Notation.}
We denote by $I_n$ the $n\times n$ identity matrix
and
\(
\bm 1_n  =  (1,\dots,1)^\tran \in \R^n.
\)
For $\bm v\in\R^n$, $\diag(\bm v)$  is the $n\times n$ diagonal matrix
with  entries $v_i$ on its diagonal.
The $i$-th row and column of a matrix $A\in\R^{n\times n}$ are denoted by $A_{[i,\cdot]}$ and $A_{[\cdot,i]}$, respectively.
We also denote $[n] = \{1,2,\dots,n\}$.
For a discrete random variable  $X$ we abbreviate 
$P(x)$ for $\Pr(X\!=\!x).$ For a second random variable \(Z\), the quantity $P(z \gn x)$ denotes either $\Pr(Z=z \gn X=x)$, or the conditional density \(p(Z=z|X=x),\) depending on whether \(Z\) is discrete or continuous.

\paragraph{Hidden Markov Models.}
Consider a discrete-time HMM
with $n$ hidden states $\{1,\dots,n\}$, whose output alphabet   $\Y $ 
is either discrete or continuous. 
Let $\F_{\prm} = \{f_{\prm}:\Y\to\R \gn \prm\in\prmset\}$  be a family of {\em parametric} probability density functions
where $\Theta$ is a suitable parameter space. 
A {\em parametric-output} HMM is defined by a tuple 
\( H = (A,\bm\prm,\initv) \)
where \(A\) is the \(n\times n\) transition matrix
of the hidden states
\[
A_{ij} = \Pr(X_{t+1} = i \gn X_{t} = j) = P(i \gn j),
\]
\(\initv\in\R^n\) is the distribution of the initial state, 
and the vector of parameters 
$\bm\prm = (\prm_1,\prm_2,\dots,\prm_n)\in\prmset^n$ 
determines the
$n$ probability density functions  $(f_{\prm_1},f_{\prm_2},\dots,f_{\prm_n})$.

The output sequence of the HMM is generated as follows.\ First, an unobserved Markov 
sequence of{ hidden states} 
$x=(x_t)_{t=0}^{T-1}$
is generated according to the  distribution
\beq
P(x) = 
\init_{x_0}\prod_{t=1}^{T-1} 
P(x_t \gn x_{t-1})
.
\eeq
The output \(y_t\in \Y\) at time $t$ depends only on the hidden state $x_t$ via 
$P(y_t\gn x_{t})\equiv f_{\prm_{x_t}}(y_t)$. Hence
the conditional distribution of an output sequence $y=(y_t)_{t=0}^{T-1}$ %$Y\in \Y^T$
is
$$
 P(y\gn x) = \prod_{t=0}^{T-1} P(y_t\gn x_t) = 
\prod_{t=0}^{T-1} f_{\prm_{x_t}}(y_t).
$$
We denote by $P_{H,k}:\Y^k \to \R$ the joint distribution of the first $k$ consecutive outputs of the HMM $H$. 
For $y=(y_0,\dots,y_{k-1}) \in \Y^k$ this distribution is given by
\beq
P_{H,k}(y) = \sum_{ x\in[n]^k }P(y \gn x) P(x)
.
\eeq
Further we denote by 
$\PH = \set{P_{H,k} \gn k\geq1}$
the set of all these distributions.

\paragraph{2-Aliased HMMs.}
For an HMM $H$
with output parameters $\prmv=(\prm_1,\prm_2,\dots,\prm_{n})\in\prmset^n$
we say that states $i$ and $j$ are {\em aliased} if $\prm_i = \prm_j$.
In this paper we consider the special case where 
$H$ has \emph{exactly two} aliased states, denoted as 2A-HMM.
Without loss of generality, we assume the aliased states are the two last ones, $n\sm1$ and $n$.
Thus,
$$
\prm_1\neq \prm_2\neq \dots\neq \prm_{n-2}\neq \prm_{n\sm1}
\quad \text{and}
\quad
\prm_{n-1} = \prm_n
.
$$
We denote the vector of the $n\sm1$ {\em unique} output parameters of $H$ by
\(
\prmotv = (\prm_1,\prm_2,\dots,\prm_{n-2},\prm_{n\sm1}) \in \prmset^{n-1}.
\)
For future use, we define the {\em aliased kernel}
$\kernel\in\R^{(n\sm1)\times (n\sm1)}$ as the matrix of inner products between the $n\sm1$ different \(f_{\prm_i}\)'s,
\beqn
\label{eq:K_def}
\kernel_{ij} \equiv
\langle f_{\prm_i},f_{\prm_j} \rangle
=
\int_{\mathcal Y} f_{\prm_i}(y)
f_{\prm_j}(y)
 \textrm{d} y,
 \quad
i,j\in[n\sm1]
 .
\eeqn

\paragraph{Assumptions.}% 
As in  previous works \citep{leroux1992maximum,kontorovich2013learning}, we make the following standard assumptions:

\begin{itemize}
	\item[({\bf A1})]
The parametric family $\F_{\prm}$ of the output distributions is linearly independent of order $n$: 
for any distinct \(\{\theta_i\}_{i=1}^n\), 
\(
\sum_{i=1}^n a_i f_{\prm_i} \equiv 0 
\)
iff
\(
a_i = 0
\)
for all
\(
i\in[n].
\)

\item[({\bf A2})]
The transition matrix $A$ is ergodic and its
unique
stationary distribution $\statv =(\pi_1,\pi_2,\dots,\pi_n)  
$
is positive. 
\end{itemize}
Note that assumption (A1) implies that the parametric family \(\F_\prm\)
is {\em identifiable}, namely $f_\prm=f_{\prm'}$ iff $\prm=\prm'$. It also implies that the 
kernel matrix $\kernel$ of (\ref{eq:K_def}) is full rank $n\sm1$.

%%%%%%%%%%%%%%%%%%%%%%%%%%%%%%%%%%%%%%%%%%%%%%%%%%%%%%%%%%
\section{Decomposing the transition matrix $A$} 
\label{sec:decompose}
The main tool in our analysis is a novel decomposition of the 2A-HMM's transition matrix into its non-aliased and aliased parts.
As shown in Lemma \ref{lem:A-exp} below, the aliased part consists of three rank-one matrices, that correspond to the dynamics of exit from, entrance to, and within the two aliased states. This decomposition is used to derive the conditions for minimality and identifiability (Section \ref{sec:minident}), and plays a  key role in learning the HMM (Section \ref{sec:learn}).

To this end, we introduce a {\em pseudo-state} $\ot$, combining the two aliased states $n\sm1$ and $n$.
We define
\beqn\pi_{\ot} = \pi_{n-1} + \pi_n
\quad\mbox{and}\quad
\beta = {\pi_{n-1}}/{\pi_{\ot}}.
        \label{eq:beta_def}
\eeqn
We shall make extensive use of
the following two matrices:
\beq
        B &=&   \left(
                        \begin{array}{ccc|cc}
                                 & & & 0 & 0 \\[-5pt]
                                 & I_{n-2} & & \vdots & \vdots    \\                    
                                 &  &  &  0  & 0\\
                                \hline
                                0 & \dots &  0  & 1 & 1   \\                            
                        \end{array}
                \right)
                \in \R^{(n\sm1) \times n}
,\\[5pt]
        C_\b &=&        \left(
                        \begin{array}{ccc|c}
                                 & & & 0  \\[-5pt]
                                 & I_{n-2} & & \vdots    \\                     
                                 &  &  &  0  \\
                                \hline
                                0 & \dots &  0  & \b   \\
                                0 & \dots &  0  &  1\sm\b  \\                           
                        \end{array}
                \right)
                \in \R^{n \times (n\sm1)}.
\eeq
As explained below,
these matrices can be viewed as projection and lifting operators, mapping between non-aliased and aliased quantities.

\paragraph{Non-aliased part.} 
The non-aliased part of \(A\) is a stochastic matrix $\bar A\in\R^{(n\sm1)\times(n\sm1)}$, obtained by {\em merging} the two aliased states $n\sm1$ and $n$ into the pseudo-state $\ot$. Its entries are given by
\beqn
\label{eq:Atilde}
\At \!  = \!
\left(
                        \begin{array}{ccc|c}
                                 & & & P(1\gn \ot)   \\
                                 & \hspace{-22pt} A_{[1:n\sm2] \times [1:n\sm2]}\hspace{-20pt}  & & \vdots    \\                       
                                 &  &  &  P(n\sm2 \gn \ot) \\[3pt]
                                \hline
                                \rule[11pt]{0pt}{0pt}
                                P(\ot\gn 1) &\hspace{-10pt} \dots &\hspace{-10pt}  P(\ot\gn n\sm2)  & P(\ot \gn \ot)                                  
                        \end{array}
                \right),
\eeqn
where the transition probabilities \textit{into} the pseudo-state are
\beq
P( \ot \gn j )  & = & P(n\sm1 \gn j ) + P(n \gn j )
,\quad \forall j\in [n],
\eeq
the transition probabilities {\em out of} the pseudo-state are defined with respect to the stationary distribution by
\beq
P(i\gn \ot) & = & \b P(i\gn n\sm1) + (1\sm\b)P(i\gn n)
,\quad \forall i\in [n]
\eeq
and lastly, the probability to {\em stay} in the pseudo-state is
\beq
P(\ot \gn \ot) & = & \b P(\ot \gn n\sm1) + (1\sm\b)P(\ot \gn n).
\eeq
It is easy to check that the unique stationary distribution of \(\At\) is  
$
\statotv  =  (\pi_1,\pi_{2}, \dots, \pi_{n\sm2} , \pi_{\ot} ) \in \R^{n\sm1}$. Finally, note that $\At  =  B A C_\b$,  
 $\statotv = B \statv$ and $\statv = C_\b \statotv$, justifying the lifting and projection interpretation of the matrices \(B,C_{\beta}\).

\paragraph{Aliased part.}
Next we present some  key quantities that distinguish between the two aliased states.
Let $\supp\row  = 
\{j\in[n] \gn P(\ot\gn j) > 0 \}$
be the set of states that can move into
either one of the aliased states. We define
\beqn
\label{eq:alp-def}
\a_j & = & 
\begin{cases}
\frac{P(n\sm1 \gn j)}{P( \ot \gn j)}
&
 j\in\supp\row
 \\
 0
 & \text{otherwise},
\end{cases}
\eeqn
as the \emph{relative probability} of moving from state $j$ to state $n\sm1$, conditional on moving to either $n\sm1$ or $n$.
We define the two vectors $\dav,\dalpv \in\R^{n-1}$ as follows:
$\forall i,j\in [n\sm1]$,
\beqn
\label{eq:deltaa}
\da_i & = &
\begin{cases}
P(i \gn n\sm1) - P(i \gn n) & i < n-1 
\\
P(\ot \gn n\sm1) - P(\ot \gn n) & i = n-1
\end{cases}
\\
\label{eq:deltaalp}
\dalp_j & =  &
\begin{cases}
(\a_{j} \sm \b) P(\ot \gn j)& j < n\sm1
\\[5pt]
\b (\a_{n\sm1} \sm \b) P(\ot \gn n\sm1) \\
\quad +\, (1\sm\b) (\a_n \sm \b) P(\ot \gn n) & j = n\sm1.
\end{cases}
\eeqn
In other words, $\dav$ captures the differences in the transition probabilities {\em out of} the aliased states. In particular, if $\dav=\bm{0}$ then starting from either one of the two aliased states, the Markov chain evolution  is identical. Intuitively such an HMM\ is not minimal, as its two aliased states can be lumped together, see Theorem \ref{thm:2-alias-irr-cond} below.  

Similarly, \(\dalpv\) compares the relative probabilities  \textit{into} the aliased states  $\a_j$, to the stationary relative probability $\b = {\pi_{n-1}}/{\pi_{\ot}}$. This quantity also plays a role in the minimality of the HMM.

Lastly, for our decomposition, we define the scalar
\beqn
\label{eq:deltaM1}
\dmomo  =  (\a_{n\sm1} - \b)P(\ot\gn n\sm1)  - (\a_n - \b) P(\ot\gn n).
\eeqn

\paragraph{Decomposing $A$.}
\newcommand{\davb}{\Delta{\bar{\bm  a}}}
\newcommand{\dalpvb}{\Delta{\bar{\bm  \a}}}
The following lemma  provides a decomposition of the transition matrix in terms of $\At$, $\dav$, $\dalpv$, $\dmomo$ and $\b$
(all omitted proofs are given in the Appendix).
\begin{lemma}
\label{lem:A-exp}
The transition matrix $A$ of a 2A-HMM can be written as
\beqn
\label{eq:A-exp}
A 
=
C_\b \At B
+
C_\b \dav
\vo_\b^\tran
+
\uo 
(\dalpv)^{\tran} B
+
\dmomo\,  \uo \vo_\b^\tran
,
\eeqn
where
\(
{\vo_\b}^{\tran}  =  (0, \dots,  0, {1\sm \b}, -\b)\in\R^n
\)
and
\(
\uo  =  (0,\dots, 0,1,-1)^{\tran} \in\R^n.
\)
\end{lemma}
\noindent
In (\ref{eq:A-exp}), the first term 
is the merged transition matrix $\At\in\R^{(n\sm1)\times(n\sm1)}$ lifted back into $\R^{n\times n}$. 
This term 
captures all of the non-aliased transitions.
The second matrix is zero except in the last two columns, accounting for 
the exit transition probabilities from the two aliased states.
Similarly, the third matrix is zero except in the last two rows, differentiating the entry probabilities. The fourth term is non-zero only on the lower right $2\times 2$ block involving the aliased states
$n\sm1$, $n$.
This term corresponds to the internal dynamics between them.
Note that each of the last three terms is at most a rank-$1$ matrix, which  together can be seen as a perturbation due to the presence of aliasing.

In section \ref{sec:learn}
we shall see that given a long output sequence from the HMM, the presence of aliasing can be detected and the quantities  $\At$, $\dav$, $\dalpv$, $\dmomo$ and $\b$ 
%(and consequently $\bar A, \davb,\dalpvb$) 
can all be estimated from it. An estimate for $A$ is then obtained via Eq. (\ref{eq:A-exp}).
 
%%%%%%%%%%%%%%%%%%%%%%%%%%%%%%%%%%%%%%%%%%%%%%%%%%%%%%%%%%%%%%%%%%%%
\section{Minimality and Identifiability}
\label{sec:minident}
Two HMMs $H$ and $H'$ are said to be {\em equivalent} if
their observed output sequences are statistically indistinguishable, namely
 $\PHt = \PH$.
Similarly, an HMM $H$ is \textit{minimal} if there is no equivalent HMM with fewer number of states. 
Note 
that if $H$ is non-aliased then Assumptions (A1-A2) readily imply that it is also minimal
\citep{leroux1992maximum}.
In this section we present necessary and sufficient conditions for a 2A-HMM to be minimal, and for two minimal 2A-HMMs to be equivalent.
% all in terms of our decomposition (\ref{eq:A-exp}).
Finally, we derive necessary and sufficient conditions for a minimal 2A-HMM to be identifiable.

\subsection{Minimality}
\label{sec:min}
The minimality of an HMM is closely related to the notion of \textit{lumpability}: can hidden states be merged without affecting the distribution $\P_H$ \citep{fredkin1986aggregated,white2000lumpable,huang2014minimal}.
Obviously, an HMM is minimal iff no subset of hidden states can be merged.
In the following theorem we give precise conditions for the minimality of a 2A-HMM.

\begin{theorem} 
\label{thm:2-alias-irr-cond}
Let $H$
be a 2A-HMM satisfying Assumptions (A1-A2) whose initial state $X_0$ is distributed according to 
\(
\initv=(\init_1, \init_2, \dots,  \beta^0\init_{\ot},(1\sm\beta^0) \init_{\ot}).
\)
Then,
\begin{itemize}%[leftmargin=10pt,topsep = 1pt, itemsep=1pt]
        \item[(i)] If  $\init_{\ot} \neq 0$ and $\b^0 \neq \b$ 
        then
        $H$ is minimal
iff $\dav \neq \bm 0$.
\item[(ii)] If 
$\init_{\ot} = 0$ or $\b^0 = \b$ then 
$H$ is minimal
iff both 
$\dav\neq~\bm 0$ and $\dalpv \neq \,\bm 0$.
\end{itemize}
\end{theorem}
\noindent
By Theorem \ref{thm:2-alias-irr-cond}, a necessary condition for minimality of a 2A-HMM is that the two aliased states have 
different exit probabilities, \(\dav\neq0\).
Namely,
there exists a non-aliased state $i\in[n\sm2]$ such that $P(i\gn n\sm1)\neq P(i\gn n)$.
Otherwise the two aliased states can be merged.
If the 2A-HMM is started from its stationary distribution, then an additional necessary condition is $\dalpv\neq0$.
This last condition implies that there is a non-aliased state $j\in\supp\row\setminus \{n\sm1,n\}$ with relative entrance probability
$\a_j \neq \b$.

\subsection{Identifiability}
\label{sec:ident}
Recall that an HMM $H$ 
is (strictly) \emph{identifiable} if $\PH$ uniquely determines the transition matrix $A$ and the output parameters $\bm\prm$, up to a permutation of the hidden states.
We establish the conditions for identifiability of a 2A-HMM in two steps.
First we derive a novel {geometric} characterization of the set of all minimal HMMs that are equivalent to $H$, up to a permutation of the hidden states 
(Theorem \ref{lem:feasreg}).
Then we give necessary and sufficient conditions for $H$ to be identifiable, namely for this set to be
the singleton set,
consisting of only $H$ itself 
(Appendix \ref{app:ident}).
In the process, we provide a simple procedure (Algorithm \ref{alg:feasreg}) to determine whether a given minimal 2A-HMM is identifiable or not.

\paragraph{Equivalence between minimal 2A-HMMs.}
Necessary and sufficient conditions for the equivalence of two minimal 
HMMs were studied in several works \citep{finesso1990consistent,ito1992identifiability, vanluyten2008equivalence}.
We now provide analogous conditions for parametric output 2A-HMMs.
Toward this end, we define 
the following 2-dimensional family of matrices $S(\tau_{n-1} , \tau_n)\in\R^{n\times n}$ 
%be the matrix 
given
 by 
\beq
                S(\tau_{n-1} , \tau_n ) =
                \left(
                        \begin{array}{ccc|cc}
                                 & & & 0 & 0  \\[-4pt]
                                 & I_{n-2} & & \vdots  & \vdots  \\                      
                                 &  &  &  0 & 0 \\
                                \hline
                                0 & \dots &  0  & \t_{n\sm1}  & \t_{n} \\
                                0 & \dots &  0  &  1\sm\t_{n\sm1} &  1\sm\t_{n}                               
                        \end{array}
                \right).
\eeq
Clearly, for $\t_{n\sm1}\neq\t_n$, $S$ is invertible. 
As in \citet{ito1992identifiability}, consider then the following similarity transformation of the transition matrix \(A,\)
\beqn
\label{eq:Atrans}
A_H(\tau_{n-1},\tau_n) &\!\! =\!\! & S(\tau_{n-1}, \tau_n)^{-1} A S(\tau_{n-1} , \tau_n).
\eeqn
It is easy to verify that $\bm{1}^\tran_n A_H=\bm{1}^\tran_n$. However, $A_H$ is not necessarily stochastic, as depending on $\t_{n-1},\t_n$ it may have negative entries. The following lemma resolves the equivalence of 2A-HMMs, in terms of this transformation. 
\begin{lemma}
\label{lem:equiv}
Let $H=(A,\prmv,\statv)$ be a minimal 2A-HMM satisfying Assumptions (A1-A2). Then a minimal HMM $H'=(A',\prmv',\statv')$ with $n'$ states is equivalent to $H$ iff $n'=n$ and 
there exists a permutation matrix $\Pi\in\R^{n\times n}$ and
$\tau_{n-1} > \tau_n $ such that 
$\prmv'   =   \Pi\, \prmv$
%$\statv' =  \Pi\, \statv_H(\t_{n-1} , \t_n)$,
and
\beq 
\statv' &=&  \Pi\, S(\t_{n-1} , \t_n)^{-1} \statv
\\
   A' & = &  \Pi\, A_H(\t_{n-1},\t_n) \,\Pi^{-1} \,\geq\, 0.
\eeq
\end{lemma}

\paragraph{The feasible region.}
By Lemma \ref{lem:equiv}, any matrix $A_H(\tau_{n-1},\t_n)$ whose entries are all non-negative yields an HMM\ equivalent to the original one. 
We thus define the {\em feasible region} of $H$ 
by
\beqn
\label{eq:feas-reg-def}
\feasreg_H = \{(\tau_{n\sm1},\tau_n)\in\R^2 \gn 
A_H(\tau_{n\sm1},\tau_n)
\geq 0,\, \tau_{n\sm1} \!>\! \tau_n\}.
\eeqn

By definition, $\Gamma_H$ is non-empty, since $(\t_{n-1},\t_n)=(1,0)$ recover the original matrix $A$. As we show below, $\feasreg_H$ is determined by three simpler regions
$\feasreg_1, \feasreg_2,\feasreg_3 \subset \R^2$.
The region $\feasreg_1$ ensures that all entries of $A_H$ are non-negative
except possibly in the lower right $2\times 2$ block corresponding to the two aliased states.
The regions $\feasreg_2$ and $\feasreg_3$ ensure non-negativity of the latter,
depending on whether the aliased relative probabilities of (\ref{eq:alp-def})
satisfy
$\a_{n\sm1} \geq \a_n$ or $\a_{n\sm1} < \a_n$, respectively.
For ease of exposition we assume as a convention that $P(\ot\gn n\sm1)\geq P(\ot\gn n)$.
\begin{theorem}
\label{lem:feasreg}
Let $H$ be a minimal 2A-HMM satisfying Assumptions (A1-A2).
There exist 
$(\t^{\min}_{n\sm1},\t^{\min}_n)$,
$(\t^{\max}_{n\sm1},\t^{\max}_n)$,
$(\t^- ,\t^+)\in\R^2$, and convex monotonic decreasing functions $\gnn,\gcc:\R\to\R$ such that
\beq
\feasreg_H = 
\begin{cases}
\feasreg_1 \cap \feasreg_2 & \a_{n\sm1} \geq \a_n \\
\feasreg_1 \cap  \feasreg_3 & \a_{n\sm1} < \a_n,
\end{cases}
\eeq
where the regions $\feasreg_1,\feasreg_2,\feasreg_3\subset \R^2$ are given by
\beq
\feasreg_1 & = &
[\t^{\min}_{n\sm1},\t^{\max}_{n\sm1}]\times[\t^{\max}_n, \t^{\min}_n ]
\\[5pt]
\feasreg_2 &=& [\t^+, \infty ) \times [\t^-,\t^+]
\\[5pt]
\feasreg_3 & = &
\{(\t_{n\sm1},\t_n)\in \feasreg_1 
\gn \,
\gnn(\t_{n\sm1}) \leq \t_n \leq \gcc(\t_{n\sm1})\,
\}.
\eeq
In addition, the set $\feasreg_H$ is connected.
\end{theorem}
\noindent
The feasible regions in the two possible cases ($\a_{n\sm1} \geq \a_n$ or $\a_{n\sm1} < \a_n$) are depicted in Appendix \ref{app:ident}, Fig.\ref{fig:feas}.

%\paragraph{Generic Identifiability.}

\paragraph{Strict Identifiability.}
By Lemma \ref{lem:equiv}, for strict identifiability of $H$,  ${\feasreg_H}$ should be the singleton set ${\feasreg_H}=\{(1,0)\}$. 
Due to lack of space, sufficient and necessary conditions for this to hold, as well as a corresponding simple procedure to determine whether a 2A-HMM is identifiable, are given in Appendix 
\ref{app:cond_feasreg}.

{\bf Remark.} 
While beyond the scope of this paper, we note that instead of strict identifiability of a given HMM, several works studied a different concept of {\em generic} identifiability
\citep{allman2009identifiability}, proving that under mild conditions the class of HMMs is generically identifiable. In contrast, if we restrict ourselves to the class of 2A-HMMs, then our Theorem \ref{lem:feasreg} implies that this class is generically \textit{non-identifiable}. The reason is that by Theorem \ref{lem:feasreg}, for any 2A-HMM\ whose matrix $A$ has all its entries positive, there are an infinite number of equivalent 2A-HMMs, implying non-identifiability.

%%%%%%%%%%%%%%%%%%%%%%%%%%%%%%%%%%%%%%%%%%%%%%%%%%%%%%%%%%%%%%%%%%%%%%%%
\section{Learning a 2A-HMM}
\label{sec:learn}
Let $(Y_t)_{t=0}^{T-1}$ be an output sequence generated by a parametric-output HMM that satisfies Assumptions (A1-A2) and initialized with its stationary distribution, $X_0 \sim \statv$.
We assume the HMM is either non-aliasing, with $n\sm1$ states, or 2-aliasing with $n$ states.
We further assume that the HMM is minimal and identifiable, as otherwise its parameters cannot be uniquely determined.

In this section we study the problems of detecting whether the HMM is aliasing and recovering its output parameters $\prmv$ and transition matrix $A$, all in terms of $(Y_t)_{t=0}^{T-1}$. 

\begin{figure}
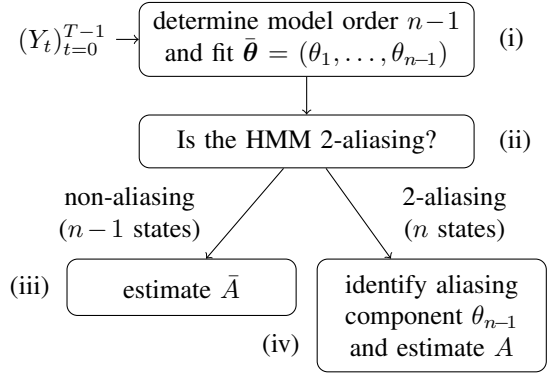
%
\centering
\tikz{
%\node [block2, text width=8em] (model) {estimate number of output components $\hat m$
% and estimate a $\hat m$-state HMM $\tilde H = (\tilde A,\prmot,\statotv)$ that best fits the data};
\node [block2, text width=12em] (nonaliased) 
{determine model order $n\sm1$ and fit $\prmotv=(\prm_1,\dots,\prm_{n\sm1})$
%% \cite{kontorovich2013learning}
};
\node [right = .2cm of nonaliased] {\centering (i)};
\node [left =.3cm of nonaliased] (observ) 
{$(Y_{t})_{t=0}^{T-1}$}
edge [->] (nonaliased);    
%edge [<-] (model);
%\node [ left =1cm of model] (observ) {$(Y_{t})_{t=0}^{T-1}$}
%edge [->] (model);    
%\node [block2, below = .5cm of nonaliased] (detect) {detect aliasing}
%edge [<-] (nonaliased);
%\node [block2, below = .5cm of model] (detect) {detect aliasing}
%edge [<-] (model);
\node [block2,text width=12em, below = .5cm of nonaliased] (detect) 
{Is the HMM 2-aliasing?}
edge [<-] (nonaliased);
\node [right = .2cm of detect] {\centering(ii)};
\node [block2red, minimum height = .75cm, below left = 1.2cm and -2.1cm of detect] (runold) 
{estimate $\At$
% as in \cite{kontorovich2013learning}
}
edge [<-] node [blockS,text width =8em ,above left = -.5cm and .0cm] 
{non-aliasing ($n\sm1$ states)} (detect)
;
\node [left = .1cm of runold] {(iii)};
\node [block2, minimum height = 1.5cm,below right = 1.2cm and -2.1cm of detect] (ident) 
{identify aliasing component $\prm_{n\sm1}$ and estimate $A$}
edge [<-] node [blockS,text width =6em , above right = -.5cm and .1cm] 
{2-aliasing ($n$ states)} (detect)
;
\node [below left = -.63cm and .1cm of ident] {(iv)};
%\node [block2, below  = .7cm of ident] (learnA) {estimate transition matrix $A$  and construct $\prmv,\statv$}
%edge [<-] (ident)
%;
%\node [block2, below = .5cm of ident] 
%{estimate $A$}
%edge [<-] (ident)
%;
}
\caption{Learning a 2A-HMM.}%
\label{fig:highlevel}%
\end{figure}
\paragraph{High level description.}
The proposed learning procedure consists of the following steps  
(see Fig.\ref{fig:highlevel}):
%(see Fig.\ref{fig:highlevel}):
\begin{itemize}%[leftmargin=*,topsep = 1pt, itemsep=1pt] 

\item[(i)]
Determine the number of output components $n\sm1$ and estimate the $n\sm1$ unique output distribution parameters $\prmotv$ and the projected stationary distribution $\statotv$.

\item[(ii)] 
Detect if the HMM is 2-aliasing.

\item[(iii)] 
In case of a non-aliased HMM, estimate the $(n\sm1)\times(n\sm1)$ transition matrix $\At$, as for example in \citet{kontorovich2013learning} or \citet{DBLP:conf/colt/Anand12}. 

\item[(iv)]
In case of a 2-aliased HMM, identify the component $\prm_{n\sm1}$ 
corresponding to the two aliased states,
and estimate the $n\times n$ transition matrix $A$.
\end{itemize}
We now describe in detail each of these steps.
As far as we know,
 our learning procedure
is the first to consistently learn a 2A-HMM in a computationally efficient way.
In particular, the solutions for problems (ii) and (iv) are new.

\paragraph{Estimating the output distribution parameters.}
As the HMM is stationary, each observable \(Y_t\) is a random realization from the following {\em parametric mixture model},
\begin{equation}
Y\sim \sum_{i=1}^{n\sm1} \statot_i f_{\prmot_i}(y).
        \label{eq:p_y}
\end{equation}
Hence, 
the number of unique output components $n-1$, the corresponding 
output parameters \(\prmotv\) and the projected stationary distribution \(\statotv\)
can be estimated by fitting a mixture model %of the form 
(\ref{eq:p_y}) to the observed output sequence \((Y_t)_{t=0}^{T-1}\).

Consistent methods to 
determine the number of components in a mixture
are well known in the literature \citep{titterington1985statistical}.
The estimation of $\prmotv$ and $\statotv$
is typically done by either applying an EM algorithm, or any recently developed spectral method \citep{dasgupta1999learning, achlioptas2005spectral,DBLP:conf/colt/Anand12}.
As our focus is on the aliasing aspects of the HMM, in what follows we assume that the number of unique output components $n\sm1$,  the output parameters $\prmotv$ and the projected stationary distribution $\statotv$ are {\em exactly known}.
As in \citet{kontorovich2013learning}, it is possible to show that our method is robust to small perturbations in these quantities (not presented).

\subsection{Moments}
To solve problems (ii), (iii) and (iv) above,
we first introduce the moment-based quantities we shall make use of.
Given $\prmotv$ and $\statotv$ or estimates of them, for any $i,j\in[n\sm1]$, we define the \emph{second order moments with time lag $t$} by 
\beqn
\label{eq:mom-def}
\momK^{(t)}_{ij} = \E[f_{\prm_i}(Y_0) f_{\prm_j}(Y_{t})]
,\quad
 t\in\{1,2,3\}.
\eeqn
The consecutive in time \emph{third order moments} are defined by
\beqn
\label{eq:momT-def}
\momTK^{(c)}_{ij} = 
\E[f_{\prm_i}(Y_0) f_{\prm_c}(Y_{1}) f_{\prm_j}(Y_{2})]
,\quad \forall c\in[n\sm1].
\eeqn
We also define the \emph{lifted kernel},
$\kernelExp = B^{\tran} \kernel B \in \R^{n\times n}.
$
One can easily verify that 
for a 2A-HMM,
\beqn
\label{eq:mom-2A}
\momK^{(t)} & =  & \kernel B A^{t} C_\b \diag(\statotv) \kernel\\
\momTK^{(c)} 
& = & \kernel B A \diag( {\kernelExp}_{[\cdot,c]} ) A C_\b \diag(\statotv) \kernel.
\eeqn
Next we define the {\em kernel free} moments $\mom^{(t)},\momT^{(c)}\in\R^{(n\sm1)\times(n\sm1)}$ as follows:
\beqn
\label{eq:striped-mom}
\mom^{(t)} &\! =\! &  \kernel^{-1} \momK^{(t)} \kernel^{-1} \diag(\statotv)^{-1}
\\
\label{eq:striped-momT}
\momT^{(c)} & \!=\! &  \kernel^{-1} \momTK^{(c)} \kernel^{-1} \diag(\statotv)^{-1}.
\eeqn
Note that by Assumption (A1),
the kernel $\kernel$ is full rank
and thus $\kernel^{-1}$ exists.
Similarly, by (A2) $\statotv > 0$, so $\diag(\statotv)^{-1}$ also exists.
Thus,  
(\ref{eq:striped-mom},\ref{eq:striped-momT}) are well defined.

Let $\dmom{2},\dmom{3},\dmomT{c}\in\R^{(n\sm1)\times (n\sm1)}$ be given
by
\begin{align}
\label{eq:deltaM2}
\dmom{2} &= \mom^{(2)} - (\mom^{(1)})^2  \\
\label{eq:deltaM3}
\dmom{3} &= \mom^{(3)} \sm \mom^{(2)} \mom^{(1)} 
 \sm \mom^{(1)}  \mom^{(2)} +
(\mom^{(1)})^3
\\
\label{eq:deltaGk}
\dmomT{c}  &= \momT^{(c)} - \mom^{(1)} \diag(\kernel_{[\cdot,c]}) \mom^{(1)}.
\end{align}
The following key lemma relates the moments (\ref{eq:deltaM2}, \ref{eq:deltaM3}, \ref{eq:deltaGk}) to the decomposition (\ref{eq:A-exp}) of the transition matrix $A$.
\begin{lemma}
\label{lem:del-rel}
Let $H$ be a minimal 2A-HMM with aliased states $n\sm1$ and $n$.
Let $\At$, $\dav$, $\dalpv$ and $\dmomo$ be defined in (\ref{eq:Atilde},\ref{eq:deltaa},\ref{eq:deltaalp},\ref{eq:deltaM1}) respectively.
Then the following relations hold:
\beqn
\label{eq:momtoA1}
\mom^{(1)} & = & \At
\\
\label{eq:dmomtoA2}
\dmom{2} & = & \dav (\dalpv)^\tran
\\
\label{eq:dmomtoA3}
\dmom{3} & = & \dmomo \dmom{2}
\\
\label{eq:dmomTtoA}
\dmomT{c} & = &
\kernel_{{n\sm1},c} \dmom{2}
,\quad \forall c\in[n\sm1].
\eeqn
\end{lemma}
\noindent
In the following, these relations will be used to detect aliasing, identify the aliased states and recover the aliased transition matrix $A$.

\paragraph{Empirical moments.}
In practice, the unknown moments  (\ref{eq:mom-def},\ref{eq:momT-def}) are estimated 
from the output sequence
$(Y_t)_{t=0}^{T-1}$ by
\beq
\emomK^{(t)}_{ij} & = & \oo{T-t}\sum_{l=0}^{T-t-1} f_{\prm_i}(Y_l) f_{\prm_j}(Y_{l+t}), \\
\emomTK^{(c)}_{ij} & = & \oo{T-2}\sum_{l=0}^{T-3} f_{\prm_i}(Y_l) f_{\prm_c}(Y_{l+1}) f_{\prm_j}(Y_{l+2}).
\eeq
With $\kernel,\statotv$ known,
the corresponding empirical kernel free moments are given by
\beqn
\label{eq:smom}
\emom^{(t)}  & = & \kernel^{-1} \emomK^{(t)}  \kernel^{-1} \diag(\statotv)^{-1}\\
\label{eq:smomT}
\emomT^{(c)}  & = & \kernel^{-1} \emomTK^{(c)}  \kernel^{-1} \diag(\statotv)^{-1}.
\eeqn
The empirical estimates for
(\ref{eq:deltaM2},\ref{eq:deltaM3},\ref{eq:deltaGk}) similarly follow.

To analyze the error between the empirical and population quantities, we make the following additional assumption:

({\bf A3}) The output distributions are bounded. Namely there exists $L>0$ such that $\forall i\in[n]$ and $\forall y\in\Y$, $f_{\prm_i}(y)\leq L$.
\begin{lemma}
\label{lem:err-asymp}
Let $(Y_t)_{t=0}^{T-1}$ be an output sequence generated by an HMM satisfying Assumptions (A1-A3).
Then, as $T\to\infty$, for any $t\in\{1,2,3\}$ and $c\in[n\sm1]$,
all error terms $\emom^{(t)} - \mom^{(t)}$, $\edmom{t} - \dmom{t}$ and
$ \edmomT{c} - \dmomT{c}$ are $O_P(T^{-\oo2})$.  
\end{lemma}
In fact, due to strong mixing, all of the above quantities are asymptotically normally distributed \citep{bradley05}.

\subsection{Detection of aliasing}
We now proceed to detect if the HMM is aliased (step (ii) in Fig.\ref{fig:highlevel}).
We pose this as a hypothesis testing problem:
\beq
\begin{array}{l}
\H_0 :\,   \text{$H$ is non-aliased with $n\sm1$ states}
\\%[-1pt]
\hspace{75pt}\text{vs.}
\\%[-1pt]
\H_1 :\,   \text{$H$ is 2-aliased with $n$ states}.
\end{array}
%\mbox{vs. }
\eeq
We begin with the following simple observation:
 \begin{lemma}
Let $H$ be a minimal non-aliased HMM with $n\sm1$ states, satisfying Assumptions (A1-A3).
Then $\dmom{2}=0$.
\end{lemma}
In contrast, 
if $H$ is 2-aliasing then according to (\ref{eq:dmomtoA2}) we have 
\(
\dmom{2}  = \dav (\dalpv)^\tran.
\)
In addition, since the HMM is assumed to be minimal and started from the stationary distribution, Theorem \ref{thm:2-alias-irr-cond} implies that both $\dav\neq 0$ and $\dalpv \neq 0$.
Thus
$\dmom{2}$ is exactly a rank-$1$ matrix, which
we write as
\beqn
\label{eq:dmom2svd}
\dmom{2} =\sigma \bm u \bm v^{\tran}
\quad \text{with} \quad
 \nrm{\u}_2 = \nrm{\v}_2 = 1,
\quad \sigma > 0,
\eeqn
where $\sigma$ is the unique non-zero singular value of $\dmom{2}$.
Hence, our hypothesis testing problem takes the form:
\beq
\H_0:\, \dmom{2} = 0
\quad
\text{vs.}
\quad
\H_1:\, \dmom{2} = \sigma \bm u \bm v^{\tran} \text{ with } \sigma > 0. 
\eeq
In practice, we only have the empirical estimate $\edmom{2}$. 
Even if $\sigma=0$, this matrix is typically full rank with $n\sm1$ non-zero singular values.
Our problem is thus detecting the rank of a matrix from a noisy version of it. There are multiple methods to do so. In this paper, motivated by \citet{kritchman2009}, we adopt the largest singular value  $\hat\sigma_1$  
 of $\edmom{2}$
 as our test statistic. The resulting test is 
\beqn
\label{eq:ths}
\text{if }  \hat\sigma_1\geq \ths \text{ return } \H_1, \text{ otherwise return } \H_0,
\eeqn
where $h_T$ is a predefined threshold. 
By Lemma \ref{lem:err-asymp}, as $T\to\infty$ the singular values of $\edmom{2}$ converge to those of $\dmom{2}$.
Thus, as the following lemma shows, with a suitable threshold this test is asymptotically consistent. 

\begin{lemma}
\label{lem:detect}
Let $H$ be a minimal HMM satisfying Assumptions (A1-A3) which is either non-aliased or 2-aliased.
Then for any $0\!<\!\eps\!<\!\oo{2}$, the test (\ref{eq:ths}) with $\ths = \Omega(T^{-\oo{2}+ \eps})$ is consistent:
% Namely, 
 as $T\to \infty$, with probability one, it will 
 correctly 
 detect whether the HMM is non-aliased or 2-aliased.
\end{lemma}

\paragraph{Estimating the non-aliased transition matrix $\At$.}
If the HMM was detected as non-aliasing,
then
its  $(n\sm1)\times(n\sm1)$ transition matrix $\At$
can be estimated 
for example by the spectral methods given in
\citet{kontorovich2013learning} or
\citet{DBLP:conf/colt/Anand12}.
It is shown there, that 
these methods are (strongly) consistent.
Moreover, as $T\to\infty$,
% namely with probability one, 
%$\lim_{ T\to\infty} ||{\eAt-\At}||_\frob = 0.$
\beqn
\label{eq:Abar-cons}
\eAt = \At + O_{P}(T^{-\oo{2}}).
\eeqn

\subsection{Identifying the aliased component $\prm_{n\sm1}$}
Assuming the HMM was detected as 2-aliasing,
our next task, step (iv), is to identify the aliased component.
Recall that if the aliased component is $\prm_{n\sm1}$, then by
 (\ref{eq:dmomTtoA})
\beq
\dmomT{c} & = &
\kernel_{{n\sm1},c} \dmom{2}
,\quad \forall c\in[n\sm1].
\eeq
We thus estimate the index $i\in[n\sm1]$ of the aliased component by solving the following least squares problem: 
\beqn
\label{eq:acestimate}
\ac = \argmin_{i\in[n\sm1]} \sum_{c\in [n\sm1]} 
\nrm{\edmomT{c} - \kernel_{i,c} \edmom{2} }_{\frob}^2.
\eeqn
The following result shows this method is consistent. 
\begin{lemma} 
\label{lem:identify}
For a minimal 2A-HMM satisfying Assumptions (A1-A3) with aliased states $n\sm1$ and $n$,
\beq
\lim_{T\to\infty} \Pr( {\ac} \neq n \sm1) = 0.
\eeq
\end{lemma}

%%%%%%%%%%%%%%%%%%%%%%%%%%%%%%%%%%%%%%%%%%%%%%%%%%%%%%%%%%%%%%%%%%%5

\subsection{Learning the aliased transition matrix $A$}
\label{sec:learnaliasA}
Given the aliased component,
we estimate the $n\times n$ transition matrix $A$ using the decomposition (\ref{eq:A-exp}).
First, recall that by (\ref{eq:dmomtoA2}),  $\dmom{2}= \dav (\dalpv)^{\tran} =  \sigma \bm u \bm v^{\tran}$. 
As singular vectors are determined only up to scaling,
 we have that\beq
\dav  = \g \bm u
\qquad \text{and} \qquad
\dalpv  =  \frac{\sigma}{\g} \bm v,
\eeq
where $\gamma\in\mathbb{R}$ is a yet undetermined constant. Thus, the decomposition (\ref{eq:A-exp}) of $A$ takes the form:
\beqn
\label{eq:A-exp-bg}
A
= 
C_{\b} \At B
+
\g C_{\b} \bm u
\vo_{\b}^\tran
+
\frac{\sigma}{\g} \uo 
\bm v^{\tran} B
+
\dmomo\,  \uo \vo_{\b}^\tran
.
\eeqn
Given that $\At, \sigma, \bm u$ and $\bm v$ are known from previous steps, we are left to determine the scalars $\g$, $\b$ and  $\dmomo$ of Eq. (\ref{eq:deltaM1}).

As for $\dmomo$,
according to (\ref{eq:dmomtoA3}) we have $\dmom{3}  =  \dmomo \dmom{2}$.
Thus,
plugging the empirical versions, $\edmomo$ is 
estimated by
\beqn
\label{eq:dmomoest}
\edmomo = \argmin_{r\in\R} \nrm{\edmom{3} - r \edmom{2}}_{\frob}^2.
\eeqn
To determine $\g$ and $\b$ we turn to 
the similarity transformation $A_H(\t_{n-1},\t_n)$, given in
(\ref{eq:Atrans}).
As shown in Section \ref{sec:decompose}, this transformation characterizes all transition matrices equivalent to $A$.
To 
relate 
$A_H$
to the form of the decomposition (\ref{eq:A-exp-bg}),
we reparametrize $\t_{n\sm1}$ and $\t_n$ as follows:
\begin{alignat*}{3}
\g' &= \g(\t_{n\sm1} - \t_n),
\qquad
&\b' &= \frac{\b - \t_n}{\t_{n\sm1} - \t_n}.
%\g_{n\sm1} &= \g (\t_{n\sm1} - \b),
%\qquad
%&\g_n &= \g (\b - \t_n)
%\\
%\g' &= \g_{n\sm1} + \g_n,
%\qquad
%&\b' &= \frac{\g_{n}}{\g_{n\sm1}+\g_n}.
\end{alignat*}
Replacing $\t_{n\sm1},\t_{n}$ with $\g',\b'$
we find that
$A_H$ is 
given by 
\beqn
\label{eq:decomp-gb}
A_H = C_{\b'} \At B
+
\g' C_{\b'} \bm u
\vo_{\b'}^\tran
+
\frac{\sigma}{\g'} \uo 
\bm v^{\tran} B
+
\dmomo\,  \uo \vo_{\b'}^\tran.
\eeqn
Note that putting 
$\g'=\g$ and $\b'=\b$ 
recovers the decomposition (\ref{eq:A-exp-bg}) for the original transition matrix $A$.

Now, since 
$H$ is assumed identifiable,
the constraint $A_H(\t_{n-1},\t_n)\geq 0$ has the unique solution $(\t_{n-1},\t_n)=(1,0)$,
or equivalently $(\g',\b')=(\g,\b)$.
Thus, with exact knowledge of the various moments, only a single pair of values  $(\g',\b')$
will yield a non-negative matrix (\ref{eq:decomp-gb}).
This perfectly recovers $\g,\b$ and the original transition matrix $A$.

In practice we plug into (\ref{eq:decomp-gb})\ the empirical versions 
$\hat \At$, $\edmomo$, $\hat\sigma_1$, $\hat {\bm u}_1$ and $\hat {\bm v}_1$,  where 
$\hat {\bm u}_1$, $\hat {\bm v}_1$ are the left and right singular vectors of $\edmom{2}$, corresponding to the singular value $\hat\sigma_1$. As described in Appendix \ref{app:estimate_gb}, the values $(\hat\g,\hat\b)$
are found by maximizing a simple two dimensional smooth function. The resulting estimate for the aliased transition matrix is
\beq
\hat A
&= &
C_{\hat\b} \hat\At B
+
\hat\g C_{\hat\b} \hat{\bm u}_1
\vo_{\hat\b}^\tran
+
\frac{\hat\sigma_1}{\hat\g} \uo 
\hat{\bm v}_1^{\tran} B
+
\hat\dmomo\,  \uo \vo_{\hat\b}^\tran
.
\eeq

The following theorem proves our method is consistent. 
\begin{theorem}
\label{thm:Ahat}
Let $H$ be a 2A-HMM satisfying assumption (A1-A3) with aliased states $n\sm1$ and $n$. 
Then 
%$\hat A$ is consistent, namely 
as $T\to \infty$,
\beq
\hat A = A +
o_P(1).
\eeq
\end{theorem}

%%%%%%%%%%%%%%%%%%%%%%%%%%%%%%%%%%%%%%%%%%%%%%%%%%%%%%%%%%
\section{Numerical simulations}
\label{sec:simul}

We present simulation results, illustrating the consistency of our methods
for the detection of aliasing, identifying the aliased component and learning the transition matrix $A$. 
As our focus is on the aliasing, we assume for simplicity that the output parameters $\prmotv$ and the projected stationary distributions $\statotv$ are exactly known.

Motivated by applications in modeling of ion channel gating \citep{crouzy1990yet,rosales2001bayesian,witkoskie2004single}, we consider the following  HMM $H$ with $n=4$ hidden states (see Fig.\ref{fig:hmm_ion}, left).
The output distributions are univariate Gaussians $\mathcal{N}(\mu_i,\sigma_i^2)$
. Its matrix 
$A$ 
and $(f_{\prm_i})_{i=1}^4$
are given by
{
\beq
A \!&\!=\!\left(\begin{array}{cccc}
  0.3 & 0.25 & 0.0 & 0.8 \\
  0.6 & 0.25 & 0.2 & 0.0 \\
  0.0 & 0.5 & 0.1 & 0.1 \\
  0.1 & 0.0 & 0.7 & 0.1
\end{array}
\right)  , 
\quad 
\begin{array}{rcl}
f_{\prm_1} &=& \mathcal{N}(3,1)\\
f_{\prm_2} &=& \mathcal{N}(6,1)\\
f_{\prm_3} &=& \mathcal{N}(0,1)\\
f_{\prm_4} &=& \mathcal{N}(0,1).
\end{array}
\eeq
States $3$ and $4$
are aliased and
by Procedure \ref{alg:feasreg}
in Appendix  \ref{sec:examples} this 2A-HMM is identifiable.
The rank-1 matrix $\dmom2$ has a singular value $\sigma=0.33$.
Fig.\ref{fig:hmm_ion} (right) shows its non-aliased version \(\Ht\) with states $3$ and $4$ merged.  

To illustrate the ability of our algorithm to detect aliasing, we generated $T$ outputs from the original aliased HMM and from its non-aliased version
$\Ht$. 
%Fig.\ref{fig:error-time} 
Fig.\ref{fig:sigma-dist} 
(left) shows the empirical densities (averaged over $1000$ independent runs) of the largest singular value of $\edmom{2}$, for both $H$ and 
$\Ht$.
In 
%Fig.\ref{fig:error-time}
Fig.\ref{fig:sigma-dist} 
(right) we show similar 
results for a 2A-HMM with $\sigma=0.22$. 
When $\sigma=0.33$, already $T=1000$ outputs suffice for essentially perfect detection of aliasing. For the smaller $\sigma=0.22$, more samples are required.

Fig.\ref{fig:mis-detiden}
%Fig.\ref{fig:error-time}
%\ref{fig:mis-detiden}
(left) shows the false alarm and misdetection  probability vs.\! sample size \(T\) 
of the aliasing detection test  (\ref{eq:ths}) with threshold  $h_T=2T^{-\oo{3}}$. The consistency of our method is evident.

Fig.\ref{fig:mis-detiden}
%Fig.\ref{fig:error-time}
%\ref{fig:mis-detiden} 
(right)
shows the probability of misidentifying the aliased component $\prm_{\bar 3}$.
We considered the same 2A-HMM $H$ but with different means for the Gaussian output distribution of the aliased states, $\mu_{{\bar 3}} = \{0,1,2\}$. 
As expected, 
%when the aliased output distribution 
when $f_{\prm_{\bar 3}}$ is closer to 
the output distribution of the non-aliased state 
$f_{\prm_1}$ (with mean $\mu_1=3$), identifying the aliased component is more difficult.

\begin{figure}[t]%
\centering
%\begin{subfigure}%{.45\textwidth}
%\centering
\tikzset{
	regS/.style = {  },
	aliasS/.style = {double = black!0, double distance=1.2pt},
%	aliasS/.style = {},
%	minustS/.style = {double = black!30, double distance=1pt}
	%double/.style = {double = black!40, double distance=1.2pt}
}
\begin{tikzpicture}
[ ->,>=stealth',shorten >=1pt,auto,node distance=1.2cm,
  %thick
  ,main node/.style={circle,fill=black!20,draw,
  font=\sffamily\small\bfseries
  },
   minimum size=.1em
   ,inner sep=2pt
   ]
  \node[main node] (1)  {1};
	\node[main node] (2) [right of=1] {2};
	\node[main node, double = black!40, double distance=1.2pt] (3) [below of = 2] {3};
  \node[main node] (4) [below of=1, double = black!40, double distance=1.2pt] {4};

  \path[every node/.style={font=\sffamily\small}]
    (1) %edge node [left] {0.6} (4)
        edge [bend left] node[above] {\tiny $.6$} (2)
				edge [bend left] node[left] {\tiny $.1$} (4)
        %edge [loop left] (1)
    (2) %edge node [right] {0.4} (1)
        %edge node {0.3} (4)
        %edge [loop right] node[right] {} (2)
        edge [bend left] node[right] {\tiny $.5$} (3)
				edge [bend left] node[above] {\tiny $.25$} (1)
    (3) %edge node [right] {0.8} (2)
    		%edge [loop right] node[below] {} (3)
        edge [bend left] node[below] {\tiny $.7$} (4)
				edge [bend left] node[right] {\tiny $.2$} (2)
    (4) %edge node [left] {0.2} (3)
        %edge [loop left] node[left] {} (4)
        edge [bend left] node[left] {\tiny $.8$} (1)
				edge [bend left] node[below] {\tiny $.1$} (3);

%[scale = 1.1,
%every node/.style = {draw, shape = circle, minimum size = 5.mm,inner sep = 0mm, fill = black!0}
%]
%\node [regS] (s1) at (-1,0) {$1$};
%\node [aliasS] (s2) at (0,1) {$2$};
%\node [regS] (s3) at (1,0) {$3$};
%\node [aliasS] (s4) at (0,-1) {$4$};
%%\draw  [dashed] (m1) -- (p1) node [left=7pt, midway, draw=none, fill = none] {$(S)$};
%%\draw  [dashed](p2) -- (m2) node [above right=18pt,draw=none, fill = none] {$(\tilde S)$};
\end{tikzpicture}
%\caption{A 2-aliased ratchet.}%
%\label{fig:ratchet}%
%\end{subfigure}
%\begin{subfigure}%{.45\textwidth}
%\centering
%\input{exampleFeas2.tex}
%\hfill
\hspace{20pt}
\newcommand{\overbar}[1]{\mkern 1.5mu\overline{\mkern-1.5mu#1\mkern-1.5mu}\mkern 1.5mu}
\tikzset{
	regS/.style = {  },
	aliasS/.style = {double = black!0, double distance=1.2pt},
%	aliasS/.style = {},
%	minustS/.style = {double = black!30, double distance=1pt}
	%double/.style = {double = black!40, double distance=1.2pt}
}
\begin{tikzpicture}
[scale = .8, ->,>=stealth',shorten >=1pt,auto ,node distance=1.2cm,
  %thick
  ,main node/.style={circle,fill=black!20,draw,
  font=\sffamily\small\bfseries
  },
   minimum size=.1em
   ,inner sep=2pt
   ]
  \node[main node] (1) {1};
	\node[main node] (2) [right = 1.2cm of 1] {2};
  \node[main node, double = black!40, double distance=1.2pt ] (3) [below right = .7cm and 0.5cm of 1] {$\overbar{\mbox{\sffamily\small\bfseries 3}}$};
  
  %\node[main node] (4) [below left of=1, double = black!40, double distance=1.2pt] {4};

  \path[every node/.style={font=\sffamily\small}]
    (1) %edge node [left] {0.6} (4)
        edge [bend left] node[above] {\tiny $.6$} (2)
				edge [bend left] node[below left] {\tiny $.1$} (3)
        %edge [loop below] node {$0.5$} (1)
    (2) %edge node [right] {0.4} (1)
        %edge node {0.3} (4)
        %edge [loop below] node[below] {\tiny $.314$} (2)
        edge [bend right] node[below right] {\tiny $.5$} (3)
        edge [] node[below] {\tiny $.25$} (1)
    (3) %edge node [right] {0.8} (2)
    		%edge [loop below] node[below] {\tiny $.3$} (3)
        edge [bend right] node[right] {\tiny $.087$} (2)
				edge [bend left] node[left] {\tiny $.453$} (1);
        % (4)
%    (4) %edge node [left] {0.2} (3)
%        edge [loop left] node[left] {$p_4$} (4)
%        edge [bend left] node[left] {$q_4$} (1);

\node [below =1.24cm of 2] {};
%[scale = 1.1,
%every node/.style = {draw, shape = circle, minimum size = 5.mm,inner sep = 0mm, fill = black!0}
%]
%\node [regS] (s1) at (-1,0) {$1$};
%\node [aliasS] (s2) at (0,1) {$2$};
%\node [regS] (s3) at (1,0) {$3$};
%\node [aliasS] (s4) at (0,-1) {$4$};
%%\draw  [dashed] (m1) -- (p1) node [left=7pt, midway, draw=none, fill = none] {$(S)$};
%%\draw  [dashed](p2) -- (m2) node [above right=18pt,draw=none, fill = none] {$(\tilde S)$};
\end{tikzpicture}
%\caption{Non-aliased merged version of \ref{fig:ratchet}.}%
%\label{fig:feas2}%
%\end{subfigure}
\caption{The aliased HMM (left) and its corresponding non-aliased version with states $3$ and $4$ merged (right).}%
\label{fig:hmm_ion}%
\end{figure}
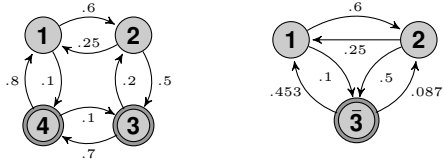

\begin{figure}[ht]%
\centering
\includegraphics[width=.5\columnwidth]{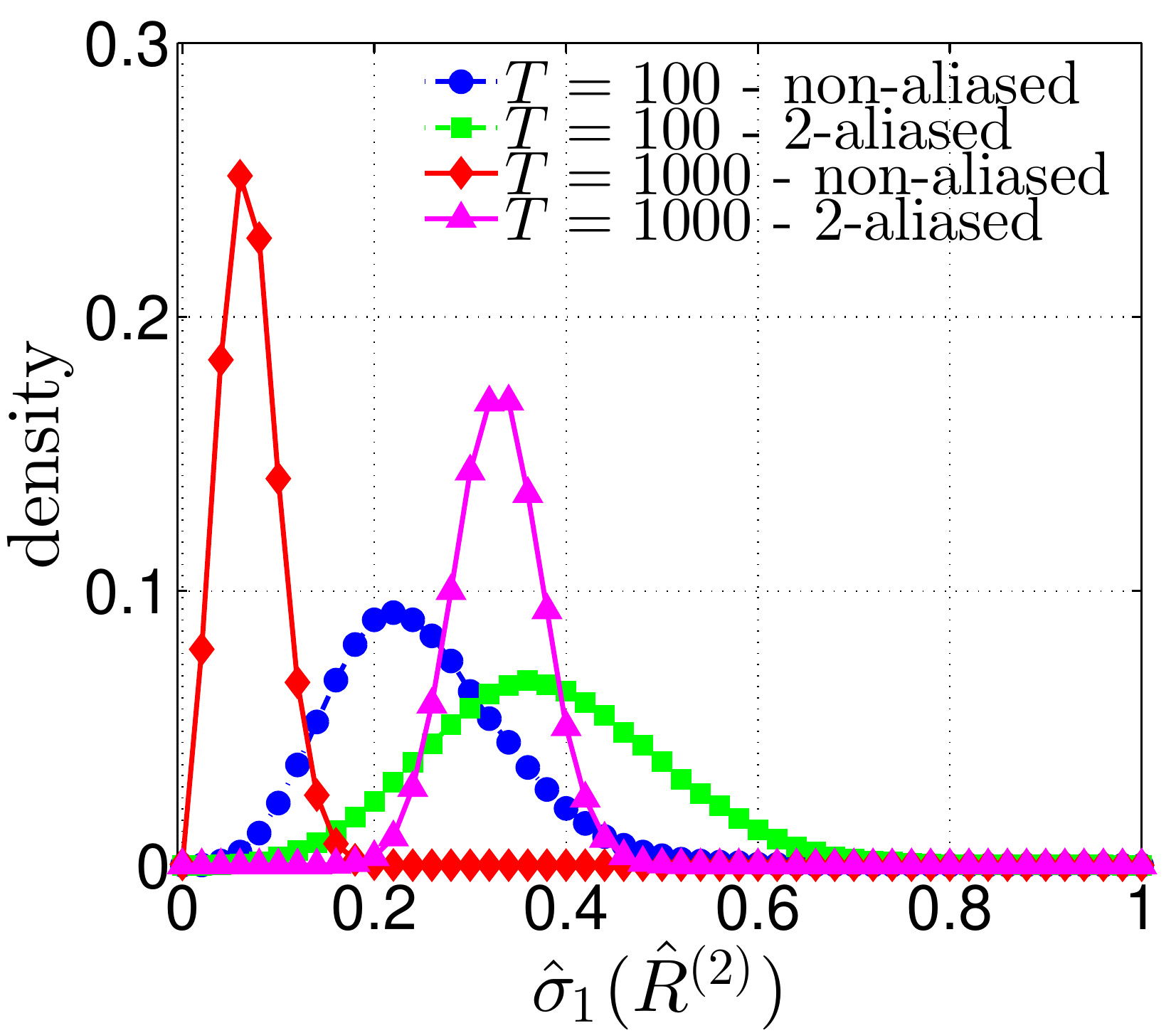}%
%\caption{$\sigma = 0.42$}
%\end{subfigure}
%\hfill
%\begin{subfigure}%{.5\columnwidth}
\includegraphics[width=.5\columnwidth]{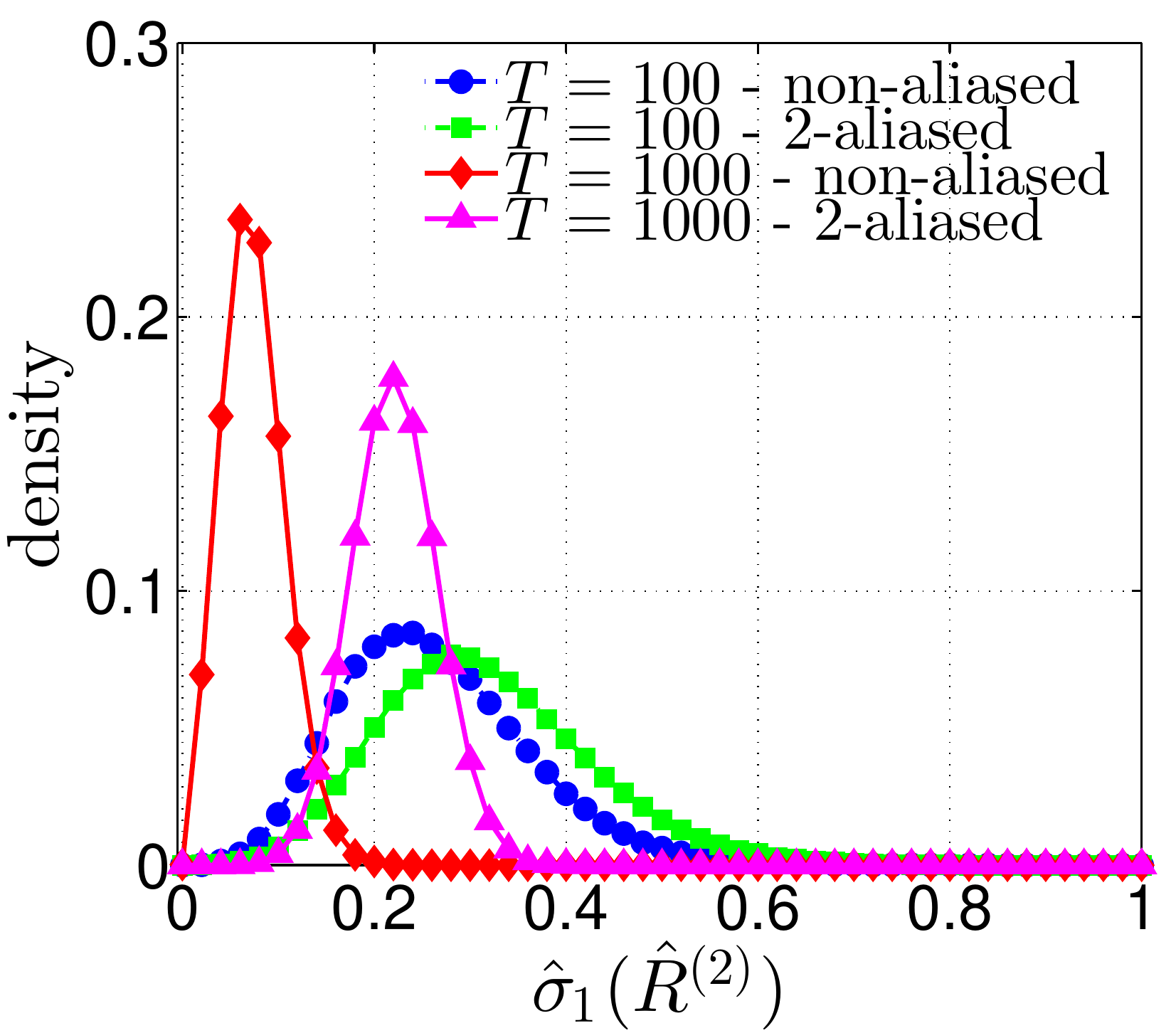}%
%%\caption{$\sigma = 0.17$}
%\end{subfigure}
\caption{Empirical density of the largest singular value of $\edmom{2}$ with $\sigma=0.33$ (left) and $\sigma=0.22$ (right).
}%
\label{fig:sigma-dist}%
\end{figure}
\begin{figure}[t]%
%\\
%\centering
%\begin{subfigure}%{.45\textwidth}
%\centering
\includegraphics[width=.5\columnwidth]{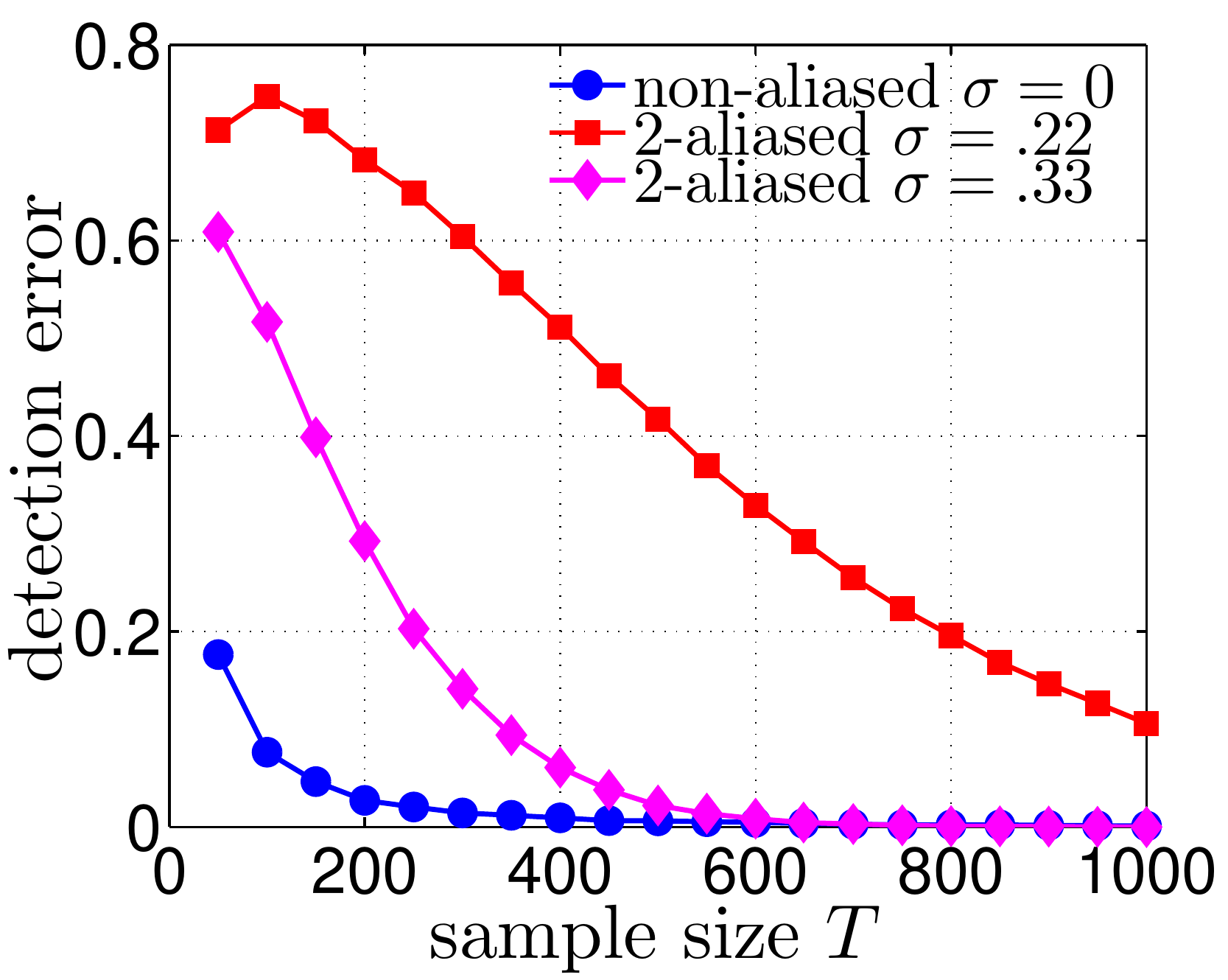}%
%\caption{}%
%\label{}%
%\end{subfigure}
%\hfill
%\begin{subfigure}%{.47\textwidth}
%\centering
\includegraphics[width=.5\columnwidth]{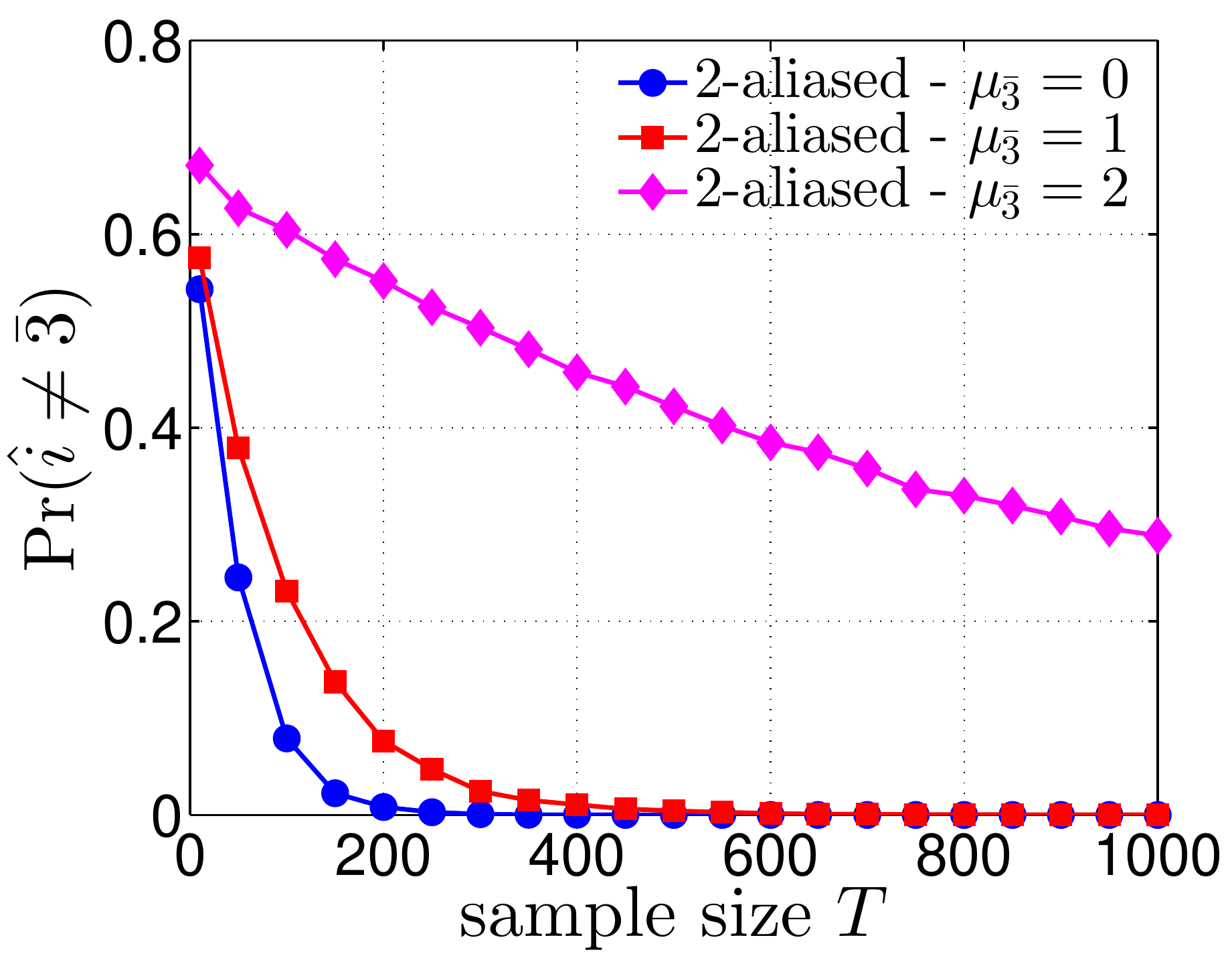}%
%\caption{}%
%\label{}%
%\end{subfigure}
\caption{%Top: Empirical density of the largest singular value of $\edmom{2}$ with $\sigma=0.42$ (left) and $\sigma=0.22$ (right).
Misdetection probability of aliasing/non-aliasing (left) and  probability of misidentifying the correct aliased component (right).
}%
\label{fig:mis-detiden}%
\end{figure}
\begin{figure}[H]%
%\\[7pt]
%\centering
\includegraphics[width=.54\columnwidth]{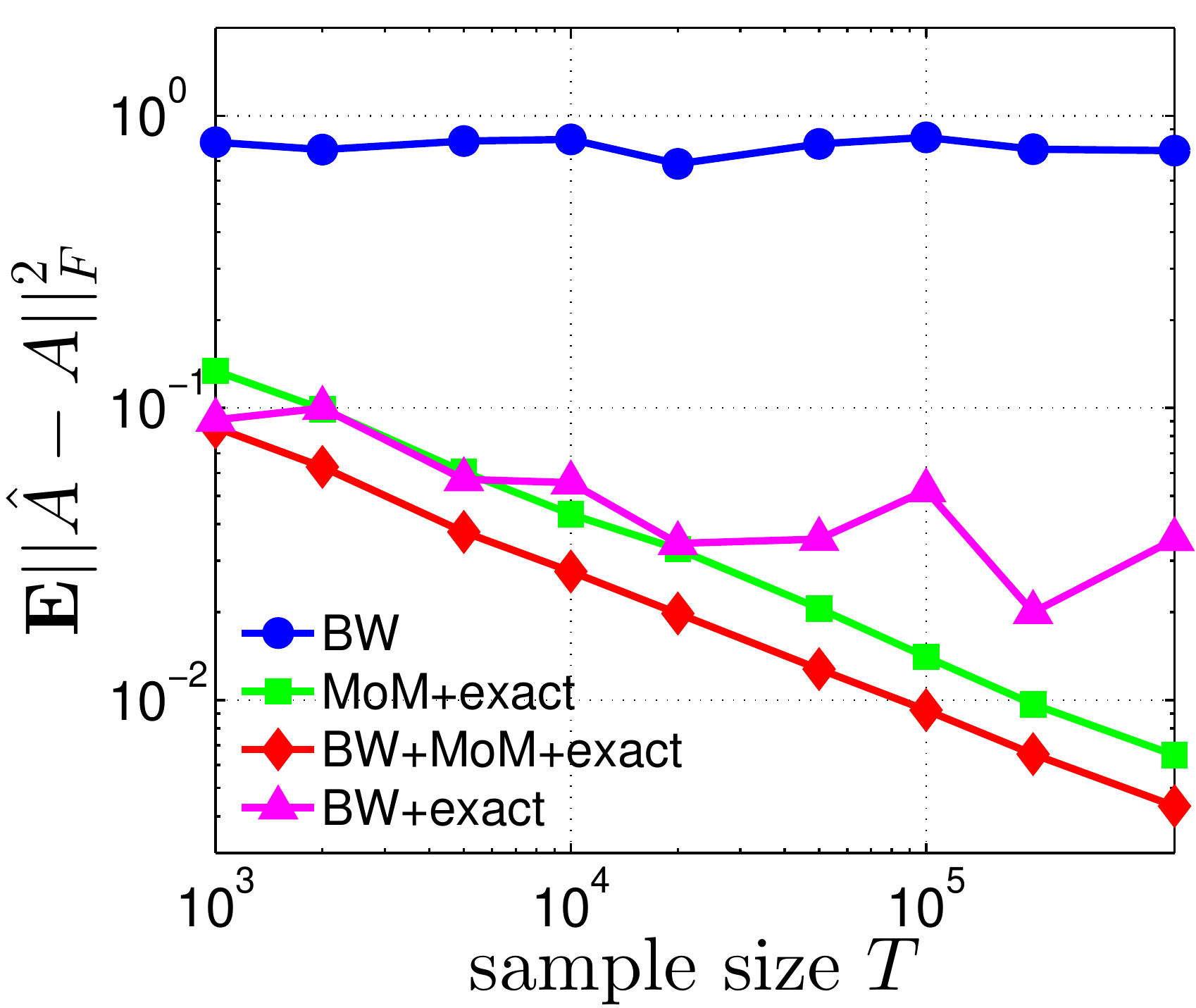}%
\hfill
\includegraphics[width=.46\columnwidth]{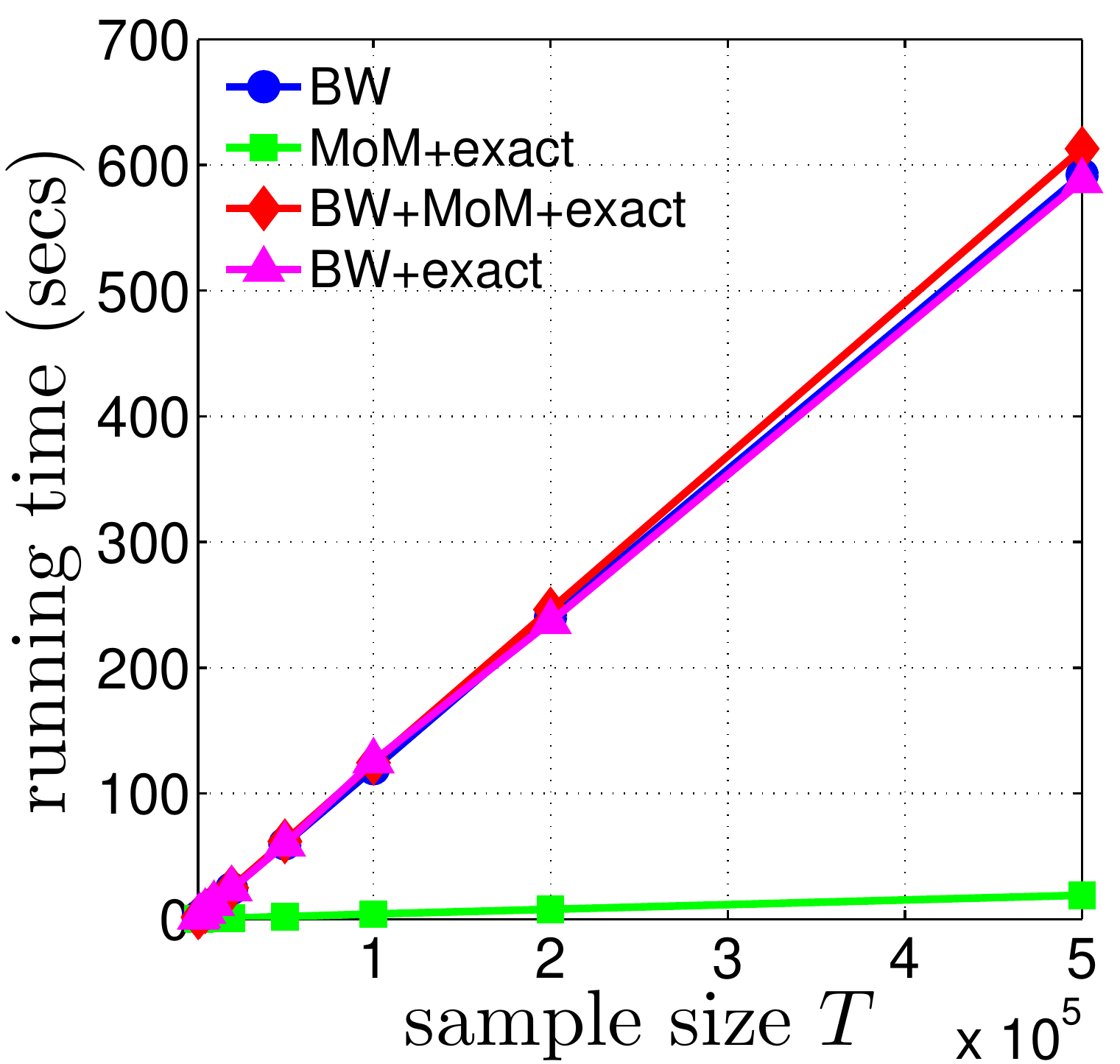}%
\caption{%{\bf Top}: Empirical density of the largest singular value of $\edmom{2}$ with $\sigma=0.33$ (left) and $\sigma=0.22$ (right).
%{\bf Middle}: Misdetection probability of aliasing/non-aliasing (left) and  probability of misidentifying the correct aliased component (right).
%{\bf Bottom}: 
Average error $\E||\hat A - A||^2_{\frob}$ and runtime comparison of different algorithms
vs. sample size $T$.
}%
\label{fig:error-time}%
\end{figure}

Finally, to estimate $A$ we considered the following methods:
The Baum-Welch algorithm with random initial guess of the HMM parameters ({BW});
our method of moments with exactly known \(\prmotv\) ({MoM+Exact});
BW initialized with the output of  our method ({BW+MoM+Exact}); 
and BW with exactly known output distributions but random initial guess of the transition matrix ({BW+Exact}).

Fig.\ref{fig:error-time} (left) shows on a logarithmic scale 
the mean square error
$\E||\hat A - A||_F^2$ vs. sample size $T$,
averaged over 100 independent realizations.
Fig.\ref{fig:error-time} (right) shows the running time as
a function of $T$. 
In these two figures, the number of
iterations of the BW was set to 20.

These results show that
with a random initial guess of the HMM parameters, BW requires far more than 20 iterations to converge. Even with exact knowledge of the output distributions  but a random initial guess for the transition matrix, BW still fails to converge after 20 iterations. In contrast, our method yields a relatively accurate estimator in only a fraction of run-time.
For an improved accuracy, this estimator can further be used as an initial guess for $A$ in the BW algorithm.

%The corresponding non-aliased HMM $\Ht$ with $3$ states is
%(see Fig.\ref{fig:hmms}, right)
%\beq
%\At \!&\!=\!\left(\begin{array}{ccc}
%  0.5 & 0.343 & 0.0  \\
%  0.5 & 0.314 & 0.7  \\
%  0.0 & 0.343 & 0.3  \\  
%\end{array}
%\right)  , 
%\quad 
%\begin{array}{rcl}
%f_{\prmot_1} &=& \mathcal{N}(-3,1)\\
%f_{\prmot_2} &=& \mathcal{N}(0,1)\\
%f_{\prmot_3} &=& \mathcal{N}(3,1)\\
%%f_{\prm_4} &=& \mathcal{N}(0,1)
%\end{array}\\
%& \statotv^\tran =   (),
%\quad \bar\sigma = 0.
%\eeq

%\newpage
\bibliographystyle{plainnat}
\bibliography{locbib}
%\bibliography{}
%\bibliographystyle{icml2015}

%%%%%%%%%%%%%%%%%%%%%%%%%%%%%%%%%%%%%%%%%%%%%%%%%%%%%%%%

%\newpage
%\quad
\newpage
\appendix

\section{Proofs for Section \ref{sec:decompose} (Decomposing $A$)}
\begin{proof}[Proof of Lemma \ref{lem:A-exp}]
\label{app:A-decompose}

Writing each term in the decomposition (\ref{eq:A-exp}) explicitly and summing these together we find a match between all entries to those of $A$.

As a representative example let us consider the last entry $A_{n,n}=P(n\gn n)$.
The first term gives
\beq
(C_\b \At B)_{[n,n]} = (1-\b) P(\ot\gn n).
\eeq
The second term gives
\beq
(C_\b \dav \vo_\b^{\tran})_{[n,n]}
&=&
-\b(1\sm\b) \da_{n\sm1}
\\
&=&
-\b(1\sm\b)(P(\ot \gn n\sm1) \sm P(\ot\gn n)).
\eeq
The third term,
\beq
(\uo (\dalpv)^{\tran} B)_{[n,n]} 
&=&
- \dalp_{n\sm1} 
\\
&=&  
-\b (\a_{n\sm1} \sm \b) P(\ot \gn n\sm1) 
\\
&&
 -\, (1\sm\b) (\a_n \sm \b) P(\ot \gn n).
\eeq
And lastly, the fourth term gives
\beq
(\dmomo \uo \vo_\b^{\tran})_{[n,n]} &=& 
\b  (\a_{n\sm1} - \b)P(\ot\gn n\sm1)  
\\
&& -\, \b(\a_n - \b) P(\ot\gn n).
\eeq

Putting $P(\ot\gn n) = P(n\sm1\gn n) + P(n\gn n)$ and $P(\ot\gn n\sm1)=P(n\sm1\gn n\sm1) + P(n\gn n\sm1)$, and summing all these four terms we obtain $P(n\gn n)$ as needed.
The other entries of $A$ are obtained similarly.
\end{proof}

\section{Proofs for Section \ref{sec:min} (Minimality)}
Let $H=(A,\bm\prm,\initv)$ be a 2A-HMM.
For any $k\geq 1$ the distribution $P_{H,k}\in\P_H$ can be cast in an explicit matrix form.
Let $o\in\Y$. The \emph{observable operator} $\obsOpr(o) \in \R^{n \times n}$ is 
defined by
\beq
\obsOpr(o) =  \diag\paren{f_{\prm_1}(o),f_{\prm_2}(o),\dots,f_{\prm_n}(o)}.
\eeq
Let $ y = (y_0,y_1,\dots,y_{k-1})\in\Y^k$ be a sequence of $k\geq1$ initial consecutive observations. 
Then the 
distribution $P_{H,k}( y)$ is given by \citet{jaeger2000observable},
\beqn
\label{eq:opr-rep}
P_{H,k}
(y)
 =  
\bm{1}_n^\tran
\obsOpr(y_{k-1})
A\,
\dots
A\,
\obsOpr(y_1)
A\,
\obsOpr(y_0)
\initv.
\eeqn

\begin{proof}[Proof of Theorem \ref{thm:2-alias-irr-cond}.]
Let us first show that $\dav\neq 0$ is necessary for minimality, namely if $\dav = 0$ then $H$ is not minimal, regardless of the initial distribution $\initv$. The non-minimality will be shown by explicitly constructing a $n\sm1$ state HMM equivalent to $H$.
Let us denote the lifting of the merged transition matrix by
\[\tilde{A} = C_\b \At B \in \R^{n\times n}.\]
Assume that $\dav = 0$.
We will shortly see that for any $\initv$ and for any $k\geq 2$ consecutive observations $y= (y_0,y_1,\dots,y_{k-1})\in\Y^k$ we have that
\beqn
\label{eq:seq-dif-span-u}
&&\obsOpr(y_{k-1}) A \dots \obsOpr(y_1) A \obsOpr(y_0) \initv
\\
\nonumber
&-& \obsOpr(y_{k-1}) \tilde{A}  \dots \obsOpr(y_1) \tilde{A}  \obsOpr(y_0) \initv
\, \propto\,{\uo}.
\eeqn
Combining (\ref{eq:seq-dif-span-u}) with (\ref{eq:opr-rep}),
%Assuming (\ref{eq:seq-dif-span-u}) holds, since 
and the fact that $\bm 1_n^\tran \uo = 0$, we have that $\P_{(A,\bm\prm,\initv)} = \P_{(\tilde{A},\bm\prm,\initv)}$. 
%Putting $A' = C \bar{A} B\in\R^{m\times m}$ one can easily check 
Since $\tilde{A}$ has identical $(n\sm1)$-th and $n$-th columns, and $f_{\prm_{n\sm1}} =f_{\prm_n}$ we have that
$\P_{(\tilde A,\bm\prm,\initv)} = \P_{(\At,\prmotv,\initot)}$. 
Thus $H' = (\At,\prmotv, \initot)$ is an equivalent $(n\sm 1)$-state HMM and $H$ is not minimal, proving the claim.
We prove  (\ref{eq:seq-dif-span-u}) by induction on the sequence length $k\geq2$.
First note that since $\dav = 0$, by Lemma \ref{lem:A-exp} we have that
\beq
A =\tilde{A} + \uo( (\dalpv)^\tran B + \dmomo \vo_\b^\tran ).
\eeq
Since for any $y\in\Y$, $\obsOpr(y) \uo = f_{\prm_{\ot}}(y) \uo$,  we have that
\beq
\obsOpr(y) A - \obsOpr(y) \tilde{A} = f_{\prm_{\ot}}(y)\uo( (\dalpv)^\tran B + \dmomo \vo_\b^\tran ) \propto\,{\uo}.
\eeq
This proves the case $k=2$.
Next, assume (\ref{eq:seq-dif-span-u}) holds for all sequences of length at least 2 and smaller than $k$,
namely,
for some $a\in\R$
\beq
&&\obsOpr(y_{k-2}) A \dots \obsOpr(y_1) A \obsOpr(y_0) \initv
\\
&=& a{\uo} + \obsOpr(y_{k-2}) \tilde{A}  \dots \obsOpr(y_1) \tilde{A}  \obsOpr(y_0) \initv
.
\eeq
Using the fact that $B \uo = 0$ we have $\obsOpr(y_{k\sm1})\tilde{A} \uo = 0$.
Inserting the expansion of $A$ in the l.h.s of 
(\ref{eq:seq-dif-span-u}) we get
\beq
&&f_{\prm_{\ot}}(y_{k\sm1}) \uo
\paren{(\dalpv)^\tran B + \dmomo \vo_\b^\tran}
\\
&\times&
\paren{
 a \uo
 +
 \tilde{A} \obsOpr(y_{k-2}) \dots \obsOpr(y_{1}) \tilde{A}  \obsOpr(y_0) \initv
 }.
\eeq
Since this last expression is proportional to $\uo$ we are done.

\paragraph{(ii) The case $\init_{\ot} = 0$ or $\b^0 = \b$.}
As we just saw,
having $\dav=0$ implies that the HMM is not minimal.
We now show that if $\dalpv = 0$ then $H$ is not minimal either.
By contraposition this will prove the first direction of (ii).

So 
assume that $\dalpv  = 0$.
Lemma
\ref{lem:A-exp} implies
\beqn
\label{eq:A-v}
A =\tilde{A} + ({C_\b} (\dav) + \dmomo \uo )\vo_\b^\tran.
\eeqn
Now note that
for all $y\in\Y$,\, $\vo_\b^\tran\obsOpr(y) = f_{\prm_{n\sm1}}(y) \vo_\b^\tran$
and since either $\init_{\ot} =0 $ or $\b^0 =\b$ we have that $\vo_\b^\tran \initv = 0$.
Thus $\vo_\b \obsOpr(y) \initv = 0$ and
 we find that
\beqn
\label{eq:exp2}
&\obsOpr(y_{k-1}) A \dots A\obsOpr(y_1) A \obsOpr(y_0) \initv 
\\
&\qquad
 = \obsOpr(y_{k-1}) A \dots A\obsOpr(y_1) \tilde{A} \obsOpr(y_0) \initv. 
\eeqn
Now since $\vo_\b^\tran C_\b = 0$ we have that for any $y\in\Y$, $\vo_\b^\tran\obsOpr(y) \tilde{A} = 0$ and
thus expanding $A$ by (\ref{eq:A-v}) we find that for any $y\in\Y$, 
\[
A \obsOpr(y) \tilde A = 
\Big(\tilde{A} + ({C_\b} (\dav) + \dmomo \uo )\vo_\b^\tran \Big) \obsOpr(y) \tilde A
=
\tilde A \obsOpr(y) \tilde A.
\]
Thus each $A$ in the right hand side of (\ref{eq:exp2}) can be replaced  by $\tilde{A}$ and we conclude that
$\P_{(A,\bm\prm,\initv)} = \P_{(\tilde{A},\bm\prm_n,\initv)}$.
Similarly to the case $\dav=0$ we have 
that 
$H' = (\At,\prmotv, \initot)$ is an equivalent $(n\sm 1)$-state HMM
and thus $H$ 
is not minimal. 

In order to prove the other direction we will show that if $H$ is not minimal then either $\dav=0$ or $\dalpv=0$. 
This is equivalent to the condition $\dav \dalpv^{\tran}=0$. 

Assuming $H$ is not minimal, there exists an HMM ${H'}$ with ${n'}<n$ states such that $\P_{H'} = \P_{H}$.
Assumptions (A1-A3) 
readily imply that $H'$ must have $n' = n\sm 1$ states
and that the unique $n\sm1$ output components are identical for $H$ and $H'$.
Since $\P_{H'}$ is invariant to permutations, we may assume that $\prmotv'=\prmotv$ and consequently the kernel matrices in (\ref{eq:K_def}) for both $H$ and $H'$ are equal $\kernel=\kernel'$.

Let $A'\in\R^{(n\sm1)\times(n\sm1)}$ be the transition matrix of $H'$ and define $H'' = (A'',\prmv'', \stat'')$ as the equivalent $n$-state HMM to $H'$ by setting $\b'' = \b$, $A'' = C_\b A' B$, $\prmv'' = \prmotv' B$ and $\statv'' = C_\b \statotv'$.
Note that for $H''$, by construction we have $\dav'' (\dalpv''){^\tran} = 0$.

Now, by the equivalence of the two models $H$ and $H''$,
we have that the second order moments ${\momK}^{(2)}$ given in (\ref{eq:mom-def}) are the same for both.
% \ref{lem:mom-matrix},
%\ref{lem:del-rel}, and 
By the fact that $\kernel''=\kernel$, $\statotv''=\statotv$ and by (\ref{eq:dmomtoA2})
 in Lemma \ref{lem:del-rel}
we must have that $\dav {(\dalpv)}^\tran = \dav'' (\dalpv''){^\tran}$.
Thus $\dav {(\dalpv)}^\tran=0$
and the claim is proved.

\paragraph{(i) The case $\init_{\ot} \neq 0$ and $\b^0\neq\b$.} 
We saw above that if $H$ is minimal then $\dav \neq 0$.
Thus, in order to prove the claim we are left to show that
if $H$ is not minimal then $\dav = 0$.

So assume $H$ is not minimal and let $H''$ be constructed as above.
By way of contradiction assume $\dav \neq 0$.
 As we just saw, since $H$ is not minimal then $\dav(\dalpv)^{\tran}$ = 0. Thus by the assumption $\dav \neq 0$
we must have $\dalpv = 0$. 
This implies that $A$ is in the form (\ref{eq:A-v}). 
Since $\P_{H} = \P_{H''}$ we have $P_{H,2}=P_{H'',2}$ where: 
\beq
P_{H,2} & = & \bm 1_n^\tran \obsOpr(y_2) A \obsOpr(y_1) \initv\\
& = & \bm 1_n^\tran \obsOpr(y_2) \Big( \tilde A +  ({C_\b} \dav + \dmomo \uo )\vo_\b^\tran \Big) \obsOpr(y_1) \initv\\
P_{H'',2} & = & \bm 1_n^\tran \obsOpr(y_2) A'' \obsOpr(y_1) \initv.
\eeq
In addition, by the fact that $\kernel''=\kernel$, $\statotv''=\statotv$ 
we must have that $\mom^{(1)} = {\mom''}^{(1)}$, 
where $\mom^{(1)}$ is defined in (\ref{eq:striped-mom}) and ${\mom''}^{(1)}$ is defined similarly with the parameters of $H''$ instead of $H$.
By (\ref{eq:momtoA1})
in Lemma \ref{lem:del-rel} we thus have 
\beq
{\mom''}^{(1)} = A' = \At = {\mom}^{(1)}.
\eeq
Hence $A'' = \tilde A$ and $P_{H,2} = P_{H'',2}$ is equivalent to
\beqn
\label{eq:PH-dif}
\bm 1_n^\tran \obsOpr(y_2) \Big({C_\b} \dav + \dmomo \uo \Big) \vo_\b^\tran  \obsOpr(y_1) \initv= 0.
\eeqn

Now, note that $\forall y_1,y_2\in\Y$ we have
\beq
\bm 1_n^\tran \obsOpr(y_2) \uo &=& 0
\\
\vo_\b^\tran \obsOpr(y_1) \initv &=& (\b^0 - \b) \init_{\ot} f_{\prm_{\ot}}(y_1)
\\
\bm 1_n^\tran \obsOpr(y_2) {C_\b} &=& %\fprmv(y_2)
(f_{\prm_1}(y_2), \dots, f_{\prm_{n\sm1}}(y_2) ).
\eeq
Thus, (\ref{eq:PH-dif}) is given by
\beq
(\b^0 - \b) \init_{\ot} f_{\prm_{n\sm1}}(y_1)
\Big(f_{\prm_1}(y_2), \dots, f_{\prm_{n\sm1}}(y_2) \Big)
 \cdot \dav = 0.
\eeq
Since by assumption $(\b^0 - \b) \init_{\ot}\neq 0$ we have $\forall y_1,y_2\in\Y$
\beq
f_{\prm_{n\sm1}}(y_1)
\Big(f_{\prm_1}(y_2), \dots, f_{\prm_{n\sm1}}(y_2) \Big)
 \cdot 
 \dav = 0.
 \eeq
For each $i\in[n\sm1]$, multiplying by $f_{\prm_i}(y_2)$ and integrating over $y_1,y_2\in\Y$ we 
get
\beq
\kernel \dav = 0.
\eeq
Since $\kernel$ is full rank we must have $\dav=0$ in contradiction to the assumption $\dav\neq0$. This concludes the proof of the Theorem.
\end{proof}

\section{Proofs for Section \ref{sec:ident} (Identifiability)}
\label{app:ident}
\begin{figure}[t]
\centering
%\begin{subfigure}
%\centering
\includegraphics[width=.4\columnwidth]{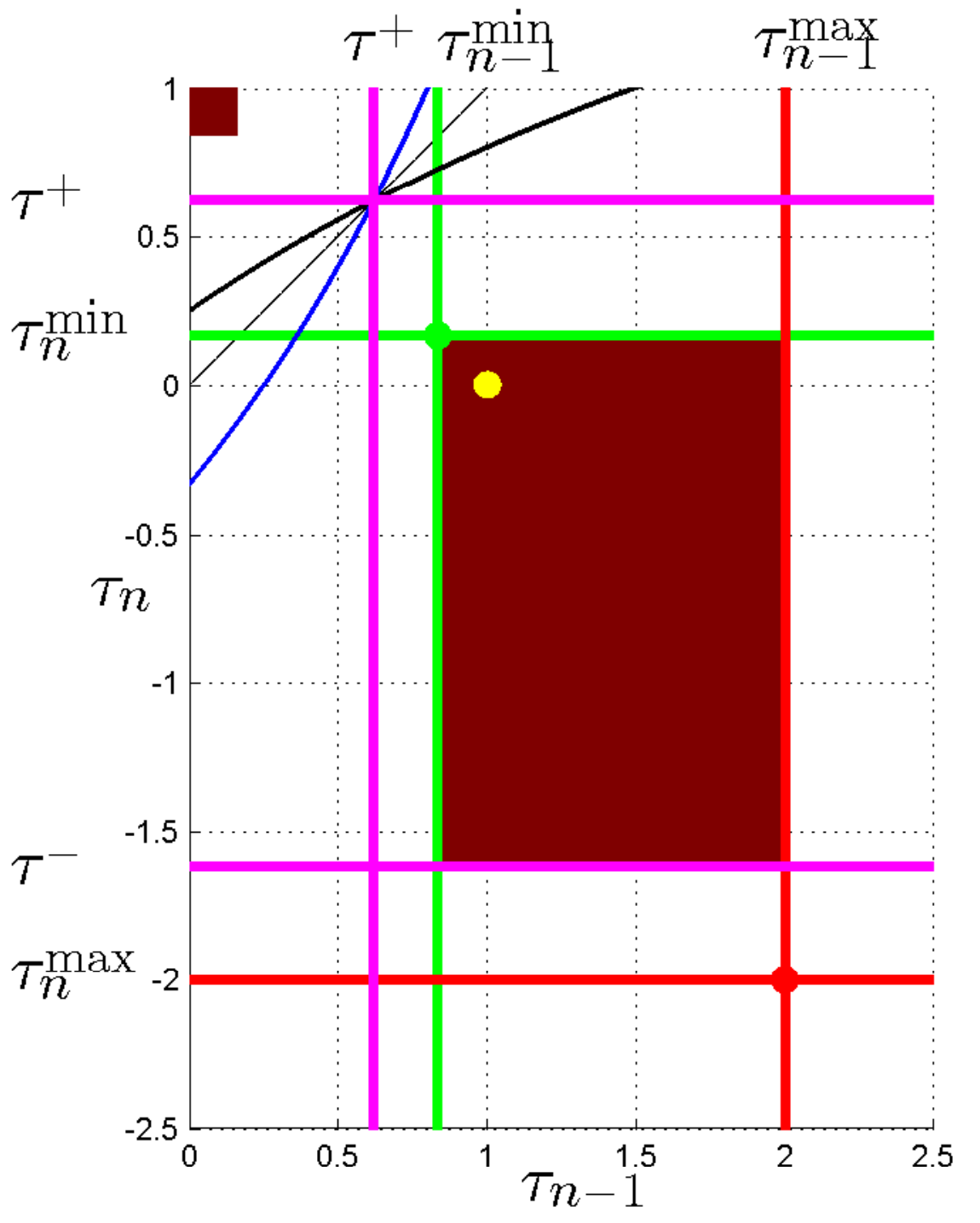}%
%\caption{$\a_{n\sm1} \geq \a_n$}
%\end{subfigure}%
%\begin{subfigure}%{10pt}
%\centering
\hspace{30pt}
\includegraphics[width=.4\columnwidth,natwidth=61,natheight=64]{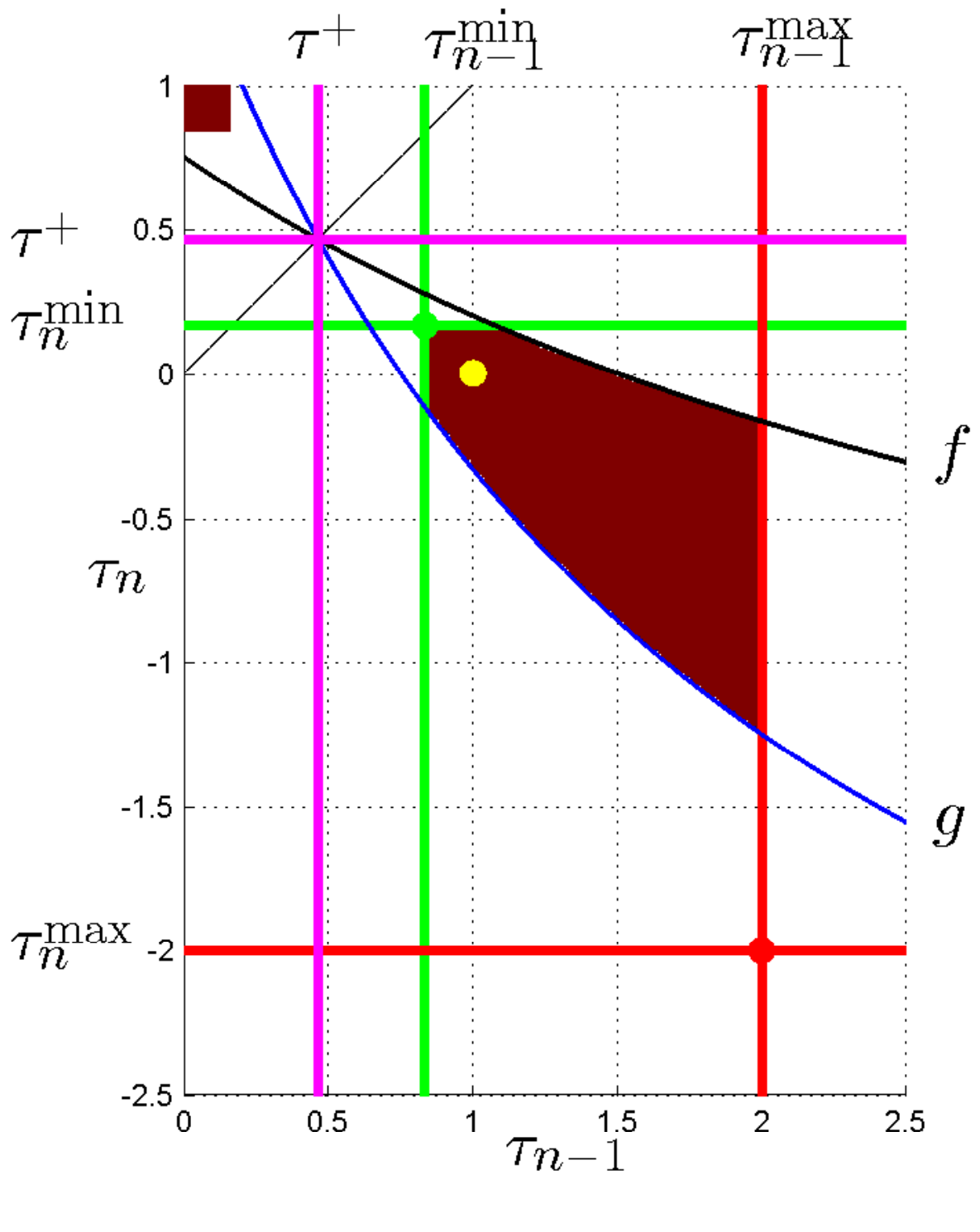}%
%\caption{$\a_{n\sm1} < \a_n$}
%\end{subfigure}
\caption{
The feasible region $\feasreg_H$ (shaded) in the $(\t_{n\sm1},\t_n)$ plane.
Any $(\t_{n\sm1},\t_n)\in\feasreg_H$ induces an HMM equivalent to $H$ via Lemma \ref{lem:equiv}.
The pair $(\t_{n\sm1},\t_n)=(1,0)$, corresponding to the original transition matrix $A$, is indicated by a yellow point. Left: $\a_{n\sm1} \geq \a_n$ and Right: $\a_{n\sm1} < \a_n$.
}
\label{fig:feas}%
\end{figure}

\subsection{Proof of Theorem \ref{lem:feasreg}}
\begin{figure}[t]%
%\centering
%\begin{subfigure}{.5\textwidth}
\centering
\usetikzlibrary{calc,through,intersections}
%\small
\begin{tikzpicture}[scale = .8]
\coordinate[label=above:$1$] (1) at (3,4.5);
\coordinate[label=right:$2$] (2) at (6,0);
\coordinate[label=left:$\bar 3$] (3) at (0,0);
\draw[name path=triangle] (1) -- (2) -- (3) -- cycle;
\coordinate (maxup) at ($(3)!.5!(1)$);
%\draw[fill=black] (maxup) circle (1.2pt) node[ left] 
%%{$\av +\g_{n\sm1}^{\max} \bm u$}
%{$\av_{\t_{n\sm1}^{\max}}$}
%;
%
\coordinate (maxdown) at ($(3)!.7!(2)$);
%\draw[fill=black] (maxdown) circle (1pt) node[below ] 
%%{$\av -\g_{n}^{\max} \bm u$}
%{$\av_{\t_{n}^{\max}}$}
%;
%
\draw [style=dashed] ($(maxup)!-.3!(maxdown)$) -- ($(maxup)!1.3!(maxdown)$);
\coordinate  (a) at ($(maxup)!.54!(maxdown)$);
\coordinate[label=right:$\av_{\bar{3}}$] (a) at ($(maxup)!.54!(maxdown)$);
\draw[fill=black] (a) circle (1pt);
\coordinate[label=above:$\av_1$] (a1) at ($(1)!.2!(2)!.1!(3)$);
\draw[fill=black,double] (a1) circle (1.5pt);
\coordinate[label=above:$\av_2$] (a2) at ($(1)!.7!(2)!.15!(3)$);
\draw[fill=black,double] (a2) circle (1.5pt);
\coordinate[label= right:$\av_3$] (a3) at ($(maxup)!.4!(a)$);
\draw[fill=black,double] (a3) circle (1.5pt);
\coordinate[label= right:$\av_4$] (a4) at ($(maxdown)!.3!(a)$);
\draw[fill=black,double] (a4) circle (1.5pt);
%
%\coordinate[label= left:
%%$\av +\g_{n\sm1}^{\min} \bm u$
%$\av_{\t_{n\sm1}^{\min}}$
%] (minup) at ($(a)!.4!(a3)$);
%\draw[fill=black] (minup) circle (1pt);
%
%\coordinate[label= left:
%%$\av -\g_{n}^{\min} \bm u$
%$\av_{\t_{n}^{\min}}$
%] (mindown) at ($(a)!.6!(a4)$);
%\draw[fill=black] (mindown) circle (1pt);
%
%\draw[|-|,thick] (maxup) -- (minup);
%\draw[|-|,thick] (maxdown) -- (mindown);
\end{tikzpicture}
\usetikzlibrary{calc,through,intersections}
%\small
\begin{tikzpicture}[scale = .8]
\coordinate[label=above:$1$] (1) at (3,4.5);
\coordinate[label=right:$2$] (2) at (6,0);
\coordinate[label=left:$\bar 3$] (3) at (0,0);
\draw[name path=triangle] (1) -- (2) -- (3) -- cycle;
\coordinate (maxup) at ($(3)!.5!(1)$);
\draw[fill=black] (maxup) circle (1.2pt) node[ left] 
%{$\av +\g_{n\sm1}^{\max} \bm u$}
{$\av_{H,{n\sm1}}^{\max}$}
;
\coordinate (maxdown) at ($(3)!.7!(2)$);
\draw[fill=black] (maxdown) circle (1pt) node[below ] 
%{$\av -\g_{n}^{\max} \bm u$}
{$\av_{H,{n}}^{\max}$}
;
\draw [style=dashed] ($(maxup)!-.3!(maxdown)$) -- ($(maxup)!1.3!(maxdown)$);
\coordinate[label=right:$\av_{\bar{3}}$] (a) at ($(maxup)!.54!(maxdown)$);
\draw[fill=black] (a) circle (1pt);
\coordinate[label=above:$\av_1$] (a1) at ($(1)!.2!(2)!.1!(3)$);
\draw[fill=black,double] (a1) circle (1.5pt);
\coordinate[label=above:$\av_2$] (a2) at ($(1)!.7!(2)!.15!(3)$);
\draw[fill=black,double] (a2) circle (1.5pt);
\coordinate[label= right:$\av_3$] (a3) at ($(maxup)!.4!(a)$);
\draw[fill=black,double] (a3) circle (1.5pt);
\coordinate[label= right:$\av_4$] (a4) at ($(maxdown)!.3!(a)$);
\draw[fill=black,double] (a4) circle (1.5pt);
\coordinate[label= left:
%$\av +\g_{n\sm1}^{\min} \bm u$
$\av_{H,{n\sm1}}^{\min}$
] (minup) at ($(a)!.4!(a3)$);
\draw[fill=black] (minup) circle (1pt);
\coordinate[label= left:
%$\av -\g_{n}^{\min} \bm u$
$\av_{H,{n}}^{\min}$
] (mindown) at ($(a)!.6!(a4)$);
\draw[fill=black] (mindown) circle (1pt);
\draw[|-|,thick] (maxup) -- (minup);
\draw[|-|,thick] (maxdown) -- (mindown);
\end{tikzpicture}
%\end{subfigure}
\caption{{\bf Top}: Plotting the columns of $BA$
on the simplex for a 2A-HMM with aliased states $\{3,4\}$. 
Here, $\av_{\bar 3} = \b \av_3 + (1-\b)\av_4$.
{\bf Bottom}: Any vectors $\av_{H,n\sm1}, \av_{H,n}$ within the depicted bars results in a matrix $A_H$ with all entries non-negative, except in possibly the $2\times2$ aliased block.
}
\label{fig:simplex}%
\end{figure}
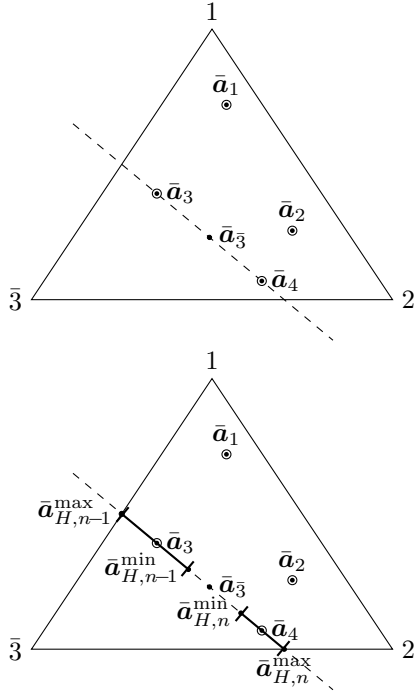

Before characterizing $\feasreg_H$ let us first give some intuition on the role of $(\tau_{n\sm1},\tau_n)$.
Consider the ${n\sm1}$ dimensional columns 
$\{\av_{i} \gn i\in[n]\}$ of the matrix $BA$.
These can be plotted on the $n\sm1$ dimensional simplex, as shown in Fig.\ref{fig:simplex} (top), for $n=4$ and aliased states $\{3,4\}$.
Recall that
$$
A_H(\t_{n\sm1},\t_n)
= S(\t_{n\sm1},\t_n)^{-1} A S(\t_{n\sm 1},\t_n)
$$
and let
$\{\av_{H,i} \gn i\in[n]\}$ be the columns of the matrix $BA_H\in\R^{{n\sm1}\times n}$.

Since $BS(\t_{n\sm1},\t_n)^{-1} = B$ we have  
\(
B A_H = B A S(\t_{n\sm1},\t_n).
\)
So the non-aliased columns of $BA_H$ are unaltered from these of $BA$, i.e. for all $i\in[n\sm2]$, $\av_{i} = \av_{H,i}$.
The new aliased columns of $BA_H$ are 
\beq
{\av_{H,n\sm1}} & = & 
%\t_c {\av_c} + (1-\t_c){\av_n} =
 {{\av}_n} + \t_{n\sm1} \dav\\
{\av_{H,n}} & = & 
%\t_n {\av_c} + (1-\t_n){\av_n}=
 {{\av}_n} + \t_n \dav.
\eeq
%where $\av_n$ is the $n$-th column of $BA$.
Thus $\t_{n\sm1}$ ($\t_n$) determines the position of the vector
$\av_{H,n\sm1}$ ($\av_{H,n}$) along the ray passing through $\av_{n\sm1}$ and $\av_n$ (dashed line in Fig.\ref{fig:simplex}).

Hence a necessary condition for $A_H$ to be a valid transition matrix is that 
${\av_{H,n\sm1}\geq0}$ and ${\av_{H,n} \geq 0}$,
%Since the non-negativity of $BA_H$ is necessary for the non-negativity of $A_H$
and one cannot take $\t_{n\sm1}$ and $\t_n$ arbitrarily.
In particular, there are $\tau^{\max}_{n\sm1}$ and $\tau^{\max}_{n}$ such that
$\av_{H,n\sm1}$ and $\av_{H,n}$ are as ``far'' apart as possible by 
putting them on the opposite sides of 
the ray connecting them, such that both 
sit on
% opposite sides of
the simplex boundary.
%intersecting the ray connecting them.
This is achieved by taking
\beq
{\av^{\max}_{H,n\sm1}} & = & 
%\t_c {\av_c} + (1-\t_c){\av_n} =
 {{\av}_n} + \t^{\max}_{n\sm1} \dav  \\
{\av^{\max}_{H,n}} & = & 
%\t_n {\av_c} + (1-\t_n){\av_n}=
 {{\av}_n} + \t^{\max}_n \dav ,
\eeq
where
\beq
\tau^{\max}_{n\sm1}  = &  \min_{j\in\X \setminus \{{n\sm1},n\} } 
\frac{\oo{2}(1 + \sign(\da_j)) - ({\av_n})_j}{\da_j} &\geq 0
\\
\tau^{\max}_n  = &  \max_{j\in\X \setminus \{{n\sm1},n\} } 
\frac{\oo{2}(1 - \sign(\da_j)) - ({\av_n})_j}{\da_j} &\leq 0.
\eeq
(see Fig.\ref{fig:simplex}, bottom).
Since we assumed as a convention that $\t_{n\sm1}> \t_n$ we have that any $\t_{n\sm1} \leq \t_{n\sm1}^{\max}$ and $\t_{n} \geq \t_{n}^{\max}$ results in a non negative matrix $BA_H$.
Note that $BA_H\geq 0$ implies ${A_H}_{[1:n\sm2,1:n]} \geq 0$.

Next, consider 
the new relative probabilities $\a_{H,i}$
as defined by (\ref{eq:alp-def}) with $A_H$ replacing $A$.
One can verify that these 
satisfy
\beq
\a_{H,i} = \frac{ \a_i - \tau_n}{\tau_{{n\sm1}} - \tau_{n}}
,\quad i\in \supp\row\backslash\{{n\sm1},n\}.
\eeq

Obviously, a necessary condition for $A_H$ to be a valid transition matrix is that
\beqn
\label{eq:alpH-const}
0 \leq \a_{H,i} = \frac{ \a_i - \tau_n}{\tau_{{n\sm1}} - \tau_{n}} \leq 1
,\quad i\in \supp\row\backslash\{{n\sm1},n\}.
\eeqn

Define the minimal and maximal relative probabilities of the non-aliased states by
\beq
\alpmin & = & \min\,  \{\a_i \gn i\in\supp\row \backslash \{{n\sm1},n\} \}\\
\alpmax & = & \max\, \{\a_i \gn i\in\supp\row \backslash \{{n\sm1},n\} \}.
\eeq
Let $\alpmin_H$ and $\alpmax_H$ be defined similarly.
Taking
\beq
\t^{\min}_{n\sm1} & = & \alpmax \\
\t^{\min}_n & = & \alpmin,
\eeq
we have $\alpmin_H = 0$ and $\alpmax_H = 1$.
Hence, for any $\t_{n\sm1} \geq \t^{\min}_{n\sm1} $ and $\t_{n} \leq \t^{\min}_{n}$ the constraint (\ref{eq:alpH-const}) holds and consequently ${A_H}_{[1:n,1:n\sm2]}$ is non-negative.
The corresponding columns $\av_{H,n\sm1}^{\min}={{\av}_n} + \t^{\min}_{n\sm1} \dav$ and $\av_{H,n}^{\min}={{\av}_n} + \t^{\min}_{n} \dav$ are depicted in
Fig.\ref{fig:simplex} (bottom).

Combining the above constraints we have that the four parameters $\t^{\min}_{n\sm1},\t^{\min}_{n}, \t^{\max}_{n\sm1},\t^{\max}_{n}$ define the rectangle
\beqn
\label{eq:feas1}
\feasreg_1 = 
[\t^{\min}_{n\sm1},\t^{\max}_{n\sm1}]\times[\t^{\max}_n, \t^{\min}_n ],
\eeqn
which characterize the equivalent matrices $A_H$
having all entries non-negative except of possibly in the $2\times 2$ aliased block (see Fig.\ref{fig:feas}). 
Thus we must have 
$\feasreg_H \subset
\feasreg_1$.

We are left to find the conditions under which the $2\times 2$ aliased block is non-negative.
Writing $A_H$ explicitly we have that these conditions are
\beqn
\label{eq:cond-nc}
A_{H,n,{n\sm1}} &= & 
\tau_{n\sm1} (\tau_{n\sm1} - \a_{n\sm1} ) P(\ot\gn{n\sm1}) %A_{n\sm1} 
\\
\nonumber
&& +\, (1-\tau_{n\sm1})(\tau_{n\sm1} - \a_n ) P(\ot\gn n) %A_n 
\geq 0
        \\[5pt]
\label{eq:cond-cn}      
A_{H,{n\sm1},n}  &=&
\tau_n (\a_{n\sm1} - \tau_n) P(\ot\gn{n\sm1}) %A_{n\sm1}
\\
\nonumber
 && +\, (1-\tau_n)(\a_n - \tau_n) P(\ot\gn n) %A_n 
\geq 0
        \\[5pt]
\label{eq:cond-cc}              
A_{H,{n\sm1},{n\sm1}}  &=&
 \tau_{n\sm1} (\a_{n\sm1} \sm \tau_n) P(\ot\gn{n\sm1}) %A_{n\sm1}  
 \\
 \nonumber
  && +\, (1\sm \tau_{n\sm1})(\a_n - \tau_n) P(\ot\gn n)  %A_n 
  \geq  0
  \\ [5pt]
\label{eq:cond-nn}        
A_{H,n,n}  &=&
 \tau_n (\tau_{n\sm1} - \a_{n\sm1}) P(\ot\gn{n\sm1}) %A_{n\sm1} 
  \\ 
  \nonumber
  && +\, (1-\tau_n)(\tau_{n\sm1} - \a_n) P(\ot\gn n) %A_n
 \geq 0.
\eeqn
As the case $P(\ot\gn{n\sm1}) = P(\ot\gn n) = 0$ is trivial, we assume that at least one of $P(\ot\gn{n\sm1}),P(\ot\gn n)$ is nonzero (and since by convention $P(\ot\gn{n\sm1}) \geq P(\ot\gn n)$, this is equivalent to 
$P(\ot\gn{n\sm1})>0$).

Recall that by definition $\da_{n\sm1} = P(\ot\gn{n\sm1}) - P(\ot\gn n)$ (see (\ref{eq:deltaa})). We now consider the cases $\da_{n\sm1}=0$ and $\da_{n\sm1}>0$ separately.

\paragraph{The case $\da_{n\sm1}=0$.}
Consider first the off-diagonal constraint (\ref{eq:cond-cn}) for $A_{H,{n\sm1},n}\geq 0$, taking the form
\beq
\t_n(1-(\a_{n\sm1}-\a_n))) \leq \a_n.
\eeq
Denote 
\[
\t^0 = {\a_n }/{(1 - (\a_{n\sm1} -\a_n))}.
\] 
Since
\(
\a_{n\sm1} - \a_n \leq 1
\)
we need
\(
 \t_n \leq 
 \t^0
.
\)
Similarly,
(\ref{eq:cond-nc}) is satisfied if and only if
\(
\t_{n\sm1} \geq 
\t^0
.
\)
Thus in order for the off-diagonal entries $A_{H,n,{n\sm1}},A_{H,{n\sm1},n}$ to be non-negative we need $(\t_{n\sm1},\t_n)\in\feasreg_2^0$ where
\beqn
\label{eq:feas20}
\feasreg_2^0 = [\t^0, \infty] \times [-\infty,\t^0].
\eeqn

Next, the on-diagonal constraint (\ref{eq:cond-cc}) for $A_{H,{n\sm1},{n\sm1}}\geq 0$ is equivalent to
\beqn
\label{eq:cons-cc-linear}
\t_n \leq \a_n + \t_{n\sm1}(\a_{n\sm1}-\a_n).
\eeqn
Similarly, the on-diagonal constrain (\ref{eq:cond-nn}) for $A_{H,n,n}\geq 0$ is
\beqn
\label{eq:cons-nn-linear}
\t_n(\a_{n\sm1}-\a_n) \leq  \t_{n\sm1} - \a_n.
\eeqn
Define the two linear functions $\gcc^0,\gnn^0: \R\to\R$ by
\beq
\gcc^0(\t_{n\sm1}) & = & \a_n + \t_{n\sm1}(\a_{n\sm1}-\a_n)\\
\gnn^0(\t_{n\sm1}) & = & \frac{\t_{n\sm1} - \a_n}{\a_{n\sm1}-\a_n}.
\eeq

Note that $\t^0$ is a fixed point of both $\gcc^0$ and $\gnn^0$,
\beq
\t^0 = \gcc^0(\t^0) = \gnn^0(\t^0).
\eeq
Note also that for $\a_{n\sm1}\geq\a_n$ the functions $\gcc^0$ and $\gnn^0$ are increasing, while for $\a_{n\sm1}<\a_n$ they are decreasing.
Thus, if $\a_{n\sm1}\geq\a_n$ the constrains (\ref{eq:cons-cc-linear},\ref{eq:cons-nn-linear})
are automatically satisfied for $(\t_{n\sm1},\t_n)\in\feasreg_2^0$, so in this case $A_{H,{n\sm1},{n\sm1}},A_{H,n,n}$ are also guaranteed to be non-negative.

If $\a_{n\sm1}<\a_n$ then with $\t_{n\sm1}\geq \t_n$ (as we assume here) we have
$\gnn^0(\t_{n\sm1}) \leq \gcc^0(\t_{n\sm1}) \leq \t_0$
and the constraints (\ref{eq:cons-cc-linear},\ref{eq:cons-nn-linear})
take the form
\(
\gnn^0(\t_{n\sm1}) \leq \t_n \leq \gcc^0(\t_{n\sm1}).
\)
Thus, in order for the on-diagonal entries $A_{H,{n\sm1},{n\sm1}}$ and $A_{H,n,n}$ to be non-negative we must have $(\t_{n\sm1},\t_n)\in\feasreg_3^0$, where
\beqn
\label{eq:feas30}
\feasreg_3^0 = \{(\t_{n\sm1},\t_n) \in \feasreg_1 \gn \gnn^0(\t_{n\sm1}) \leq \t_n \leq \gcc^0(\t_{n\sm1}) \}.
\eeqn

We are left to ensure that for  $\a_{n\sm1}<\a_n$ 
the off diagonal entries 
are also non-negative.
Indeed, since $\t_n \leq \t_{n\sm1}$, $\t^0$ is a fixed point and $\gcc(\t_{n\sm1}),\gnn(\t_{n\sm1})$ are decreasing, for any $(\t_{n\sm1},\t_n)\in\feasreg_3^0$
we automatically have that  $\t_0 \leq \t_{n\sm1}$ and $\t_n \leq \t_0$, 
so $(\t_{n\sm1},\t_n)\in\feasreg_3^0$ implies $(\t_{n\sm1},\t_n)\in\feasreg_2^0$.
Thus all entries of the aliasing block are guaranteed to be non-negative.

To conclude, we have shown that for $\da_{n\sm1}=0$ the feasible region (\ref{eq:feas-reg-def}) is given by
\beq
\feasreg_H^0 = 
\begin{cases}
\feasreg_1 \cap \feasreg_2^0 & \a_{n\sm1} \geq \a_n \\
\feasreg_1 \cap  \feasreg_3^0 & \a_{n\sm1} < \a_n.
\end{cases}
\eeq

\paragraph{The case $\da_{n\sm1}>0$.}
This case has the same characteristics as for the $\da_{n\sm1}=0$ case, but it is a bit more complex to analyze.
Define $\t^{\pm}$ (as the analogues of $\t^0$) by
\beqn
\label{eq:quad-ineq}
\t^{\pm} &=& \frac{1}{
2\da_{n\sm1}
}
\Big(
\a_{n\sm1} P(\ot\gn{n\sm1}) 
\\
\nonumber
&& \hspace{50pt}
- (1+\a_n) P(\ot\gn n)
 \pm {\sqrt{\Delta}}
 \Big),
\eeqn
where
\beq
{\Delta} 
& = & 
\Big(\a_{n\sm1} P(\ot\gn{n\sm1}) - (1+\a_n)P(\ot\gn n) \Big)^2 
\\
&& \hspace{80pt}
+\, 4 \a_n P(\ot\gn n) \da_{n\sm1}
\geq 0.
\eeq
And define the regions
\beqn
\label{eq:feas2}
\feasreg_2 &\!\!= \!\!& [\t^+, \infty ] \times [\t^-,\t^+]
\\[5pt]
\label{eq:feas3} 
\feasreg_3  &\!\!= \!\!&
\{(\t_{n\sm1},\t_n)\in \feasreg_1 
 \gn \,
 \gnn(\t_{n\sm1}) \leq \t_n \leq \gcc(\t_{n\sm1})
  \}
\eeqn
where the functions $\gcc,\gnn:\R\to\R$ are given by
\beqn
\label{eq:gcc-def}
\gcc(\t_{n\sm1}) & = & \paren{\frac{\a_{n\sm1} P(\ot\gn{n\sm1}) - \a_n P(\ot\gn n)}{\da_{n\sm1}}}
\\
\nonumber
&& -\,  \frac{P(\ot\gn{n\sm1}) P(\ot\gn n) (\a_{n\sm1}-\a_n)}{(\da_{n\sm1})^2} 
\\
\nonumber
&& \times\,
\paren{\t_{n\sm1} - \Big(\frac{-P(\ot\gn n)}{\da_{n\sm1}}\Big)}^{-1}
\eeqn
and
\beqn
\label{eq:gnn-def}
\gnn(\t_{n\sm1}) & = & 
\paren{\frac{-P(\ot\gn n)}{\da_{n\sm1}}}
\\
\nonumber
&&
-\, \frac{P(\ot\gn{n\sm1}) P(\ot\gn n) ( \a_{n\sm1} - \a_n )}{(\da_{n\sm1})^2}
\\
\nonumber
&& \hspace{-50pt} \times\,
\paren{\t_{n\sm1}- \Big(\frac{\a_{n\sm1}P(\ot\gn{n\sm1}) - \a_n P(\ot\gn n)}{\da_{n\sm1}}\Big)}^{-1}.
\eeqn

\begin{lemma} 
\label{thm:feasreg}
Let $\feasreg_1$ be defined in (\ref{eq:feas1}) and let $\feasreg_2$ and $\feasreg_3$ be defined according to whether $\da_{n\sm1}=0$ (\ref{eq:feas20},\ref{eq:feas30}) or not (\ref{eq:feas2},\ref{eq:feas3}). 
Then the feasible region $\feasreg_H$ satisfies
\beq
\feasreg_H = 
\begin{cases}
\feasreg_1 \cap \feasreg_2 & \a_{n\sm1} \geq \a_n \\
\feasreg_1 \cap \feasreg_3 & \a_{n\sm1} < \a_n.
\end{cases}
\eeq
\end{lemma}

\begin{proof}
As the case $\da_{n\sm1}=0$ was treated above, we consider the case $\da_{n\sm1}>0$.
Consider first the off diagonal constraint (\ref{eq:cond-cn}).
Multiplying by $(-1)$, we need to solve
the following  inequality for $\t\in\R$,
\beqn
\label{eq:t-quad}
&&\t^2 
\da_{n\sm1}
 - \t(\a_{n\sm1} P(\ot\gn{n\sm1})
 \\
 \nonumber
&& \hspace{20pt}
  - (1+\a_n)P(\ot\gn n)) -\a_n P(\ot\gn n) \leq 0.
\eeqn
We
first solve with equality to find the solutions $\t^-,\t^+$ given in (\ref{eq:quad-ineq}).
Thus, since $\Delta\geq 0$ we have that any feasible $\t_n$ must satisfy 
\(
\t^- \leq \t_n \leq \t^+.
\)
Note that the constraint (\ref{eq:cond-nc}) for $\t_{n\sm1}$ is the complement of 
(\ref{eq:cond-cn}),
 and by assumption $\t_{n\sm1} \geq \t_n$, 
so (\ref{eq:cond-nc}) is satisfied iff $\t^+ \leq \t_{n\sm1}$.
Thus, the region $\feasreg_2$ given in 
(\ref{eq:feas2}) indeed
characterize the non-negativity of both $A_{{n\sm1},n}$ and $A_{n,{n\sm1}}$.
With some algebra, $\t^+$ and $\t^-$ can be shown to satisfy the following useful relations:

\begin{itemize}
        \item If $\a_{n\sm1} \geq \a_n$ then
\beqn
\label{eq: tm-a-geq}
-\frac{P(\ot\gn{n\sm1})}{\da_{n\sm1}}
\leq \t^- 
\leq 0
\eeqn
and 
\beqn
\label{eq: tp-a-geq}
0 \leq \t^+ 
\leq \frac{\a_{n\sm1} P(\ot\gn{n\sm1}) - \a_n P(\ot\gn n)}{\da_{n\sm1}}.
\eeqn
\item If $\a_{n\sm1} < \a_n $ then
\beqn
\label{eq: t-a-leq}
\t^- 
&\leq& -\frac{P(\ot\gn n)}{\da_{n\sm1}}
\\
\nonumber
&\leq& \frac{\a_{n\sm1} P(\ot\gn{n\sm1}) - \a_n P(\ot\gn n)}{\da_{n\sm1}} 
\leq \t^+.
\eeqn
\end{itemize}

We proceed to handle the constraints (\ref{eq:cond-cc}) and (\ref{eq:cond-nn})
 corresponding to the region $\feasreg_3$.
We begin by solving the inequality (\ref{eq:cond-cc}):
\beq
&& -\t_n( P(\ot\gn n) +\t_{n\sm1} \da_{n\sm1}) + \a_n P(\ot\gn n) \\
&& \hspace{20pt} +\, \t_{n\sm1}( \a_{n\sm1}P(\ot\gn{n\sm1})- \a_n P(\ot\gn n)) \geq 0.
\eeq
Note that for $(\t_{n\sm1},\t_n) \in \feasreg_1$ we have 
%$\t_c \geq -A_{\ot,n} /\da_{\ot}$ (as otherwise $A'_{\ot,c} < 0$) and thus 
$(P(\ot\gn n) + \t_{n\sm1}\da_{n\sm1}) \geq 0$. 
Rearranging we get that 
in order for $A_{H,{n\sm1},{n\sm1}}$ to be non-negative we must have that
\beqn
\label{eq:tn-gcc}
\text{ if } \quad
\t_{n\sm1} \geq -\frac{P(\ot\gn n)}{\da_{n\sm1}}
\quad
\text{then}
\quad
 \t_n \leq \gcc(\t_{n\sm1}),
\eeqn
where $\gcc$ is the function given in (\ref{eq:gcc-def}).
Similarly, consider the condition (\ref{eq:cond-nn}),
\beq
&& \t_n\Big( \a_n P(\ot\gn n) - \a_{n\sm1} P(\ot\gn{n\sm1})
 + \t_{n\sm1}\da_{n\sm1}\Big) 
\\
&& \quad +\, P(\ot\gn n)( \a_n - \t_{n\sm1}) \,\geq\, 0.
\eeq
Rearranging we find that in order for $A_{H,n,n}\geq 0$ we must have
\beqn
\label{eq:tn-gnn}
\begin{cases}
\t_n \leq \gnn(\t_{n\sm1}) & \t_{n\sm1} \leq \frac{\a_{n\sm1} P(\ot\gn{n\sm1})- \a_n P(\ot\gn n)}{P(\ot\gn{n\sm1}) - P(\ot\gn n)}\\
\t_n \geq \gnn(\t_{n\sm1}) & \text{otherwise},
\end{cases}
\eeqn
where
the function $\gnn$ 
 is given 
 in (\ref{eq:gnn-def}).
Note that $\gcc$ (res. $\gnn$) defines the boundary where (\ref{eq:cond-cc}) (res. (\ref{eq:cond-nn})) changes sign, namely any pair 
$(\t_{n\sm1},\t_n) = (\t_{n\sm1},\gcc(\t_{n\sm1}))$ is on the curve making Equation (\ref{eq:cond-cc}) equal zero,
and
similarly $\gnn(\t_{n\sm1})$ is  
such that 
$(\t_{n\sm1},\t_n) = (\t_{n\sm1},\gnn(\t_{n\sm1}))$ is on the curve making 
(\ref{eq:cond-nn}) equal zero. 
Having the boundaries $\gcc,\gnn$ in our disposal let us first consider the case $\a_{n\sm1}\geq\a_n$.

\paragraph{The sub-case $\a_{n\sm1}\geq\a_n$.} 
We show that in this case, having $(\t_{n\sm1},\t_n)\in\feasreg_1 \cap \feasreg_2$ already ensures 
that conditions (\ref{eq:tn-gcc}) and (\ref{eq:tn-gnn})
are trivially met,
which in turn implies the
non-negativity of both $A_{H,{n\sm1},{n\sm1}}$ and $A_{H,n,n}$.
This 
is done
 by showing that for any $(\t_{n\sm1},\t_n)\in\feasreg_1 \cap \feasreg_2$ the curve $(\t_{n\sm1}, \gcc(\t_{n\sm1}))$ is above $(\t_{n\sm1},\t^+)$.
Similarly, for 
$\t_{n\sm1} < {(\a_{n\sm1} P(\ot\gn{n\sm1}) - \a_n P(\ot\gn n))}/{\da_{n\sm1}}$ the curve $(\t_{n\sm1}, \gnn(\t_{n\sm1}))$ is above $(\t_{n\sm1},\t^+)$ and for $\t_{n\sm1} > {(\a_{n\sm1} P(\ot\gn{n\sm1})- \a_n P(\ot\gn n))}/{\da_{n\sm1}}$ the curve $(\t_{n\sm1}, \gnn(\t_{n\sm1}))$ is below $(\t_{n\sm1},\t^-)$,
thus making conditions (\ref{eq:tn-gcc}) and (\ref{eq:tn-gnn}) true.
Toward this end consider
the equality $\gcc(\t) = \gnn(\t)$ given by
\beq
0 & = &  \Big(\a_{n\sm1} P(\ot\gn{n\sm1}) + (1-\a_n) P(\ot\gn n)\Big) \times
 \\
&&\Big(\t^2 \da_{n\sm1}
 + \t((1+\a_n)P(\ot\gn n)
\\
&&\hspace{20pt}
- \a_{n\sm1} P(\ot\gn{n\sm1}))
   -\a_n P(\ot\gn n)\Big).
 \eeq
Thus if $\paren{\a_{n\sm1} P(\ot\gn{n\sm1}) + (1-\a_n) P(\ot\gn n)}=0$ we have that $\gcc=\gnn$ identically. 
Otherwise we need to solve again (\ref{eq:t-quad}) so the solutions are $\t^+,\t^-$ with $\gcc(\t^+) = \gnn(\t^+)$ and $\gcc(\t^-) = \gnn(\t^-)$.
In addition 
one can show that $\t^+$ and $\t^-$ are in fact fixed points of both $\gcc$ and $\gnn$,
so together we have
\beq
&&\t^+ = \gcc(\t^+) = \gnn(\t^+) \\
&&\t^- = \gcc(\t^-) = \gnn(\t^-).
\eeq
Inspecting $\gcc(\t_{n\sm1})$ one can see that for $\t_{n\sm1} \geq -P(\ot\gn n)/\da_{n\sm1}$, 
$\gcc(\t_{n\sm1})$  is monotonic increasing and concave.
Since by (\ref{eq: tp-a-geq}) we have $\t^+ \geq -P(\ot\gn n)/\da_{n\sm1}$ 
we get that for $\t_{n\sm1} \geq \t^+$ we must have $\gcc(\t_{n\sm1}) \geq \t^+$ as needed.
Similarly, for $\t^+ \leq \t_{n\sm1} < {(\a_{n\sm1} P(\ot\gn{n\sm1}) - \a_n P(\ot\gn n))}/{\da_{n\sm1}}$ the function $\gnn(\t_{n\sm1})$ is increasing and convex and thus above $\t^+$, while for $ {(\a_{n\sm1}P(\ot\gn{n\sm1}) - \a_n P(\ot\gn n))}/{\da_{n\sm1}}< \t_{n\sm1}$ it is increasing but always below $\t^-$.
Thus, for $\a_{n\sm1} \geq \a_n$ we have that $\feasreg_1 \cap \feasreg_2$ also characterize the non-negativity of $A_{H,{n\sm1},{n\sm1}}$ and $A_{H,n,n}$ as claimed.

\paragraph{The sub-case $\a_{n\sm1} <\a_n$.}
Note that by (\ref{eq: t-a-leq}) we have $\t^+ \geq {(\a_{n\sm1} P(\ot\gn{n\sm1}) - \a_n P(\ot\gn n))}/{\da_{n\sm1}}$.
Thus for $\t^+ \leq \t_{n\sm1}$ both $\gcc$ and $\gnn$ are decreasing and convex  and $\gnn(\t_{n\sm1}) \leq \gcc(\t_{n\sm1})$.
Thus in order to ensure (\ref{eq:tn-gcc}, \ref{eq:tn-gnn}) we need 
\(
\gnn(\t_{n\sm1}) \leq \t_n \leq \gcc(\t_{n\sm1}).
\)
Thus, $\feasreg_3 $ as defined in (\ref{eq:feas3}) characterize the non-negativity of $A_{H,n,n},A_{H,{n\sm1},{n\sm1}}$.
Finally we need to show that having $(\t_{n\sm1},\t_n)\in \feasreg_3$ also ensures the non-negativity of $A_{H,n,{n\sm1}}$ and $A_{H,{n\sm1},n}$. But by (\ref{eq: t-a-leq})  we have that for $\t_{n\sm1} \geq \t^+$ both $\gcc(\t_{n\sm1}),\gnn(\t_{n\sm1}) \geq \t^-$ and thus $\feasreg_3 \subset \feasreg_2$.
Hence we have shown that $\feasreg_H$ is characterized as claimed.
\end{proof}

\begin{lemma}
\label{lem:conn}
 The set $\feasreg_H$ is connected.
\end{lemma}
\noindent
\begin{proof}
If $\a_{n\sm1} \geq \a_n$ then $\feasreg_H = \feasreg_1 \cap \feasreg_2$ is a rectangle and thus connected. In the case $\a_{n\sm1} < \a_n$ we have that 
$\gnn(\t_{n\sm1}) \leq \gcc(\t_{n\sm1})$ and are both decreasing and convex thus the region $\feasreg_3$ with intersection with a rectangle is a connected set.
\end{proof}

%\begin{figure}
%\small
{%\centering
\begin{table}%{.6\columnwidth}
\centering
{%\scriptsize
\begin{tabular}{c|c|c|c}
&
$A_{i,{n\sm1}}\!=\!0$
&
$0\!<\!A_{i,{n\sm1}}\!<\!1$
&
$A_{i,{n\sm1}}\!=\!1$
\\
%%%%%%%%% Second row %%%%%%%%%%%%%%%%%%%%%%%%%%%5
\hline
\parbox[c]{4em}{$A_{i,n}\!=\!0$}
&
\parbox[c][5em][c]{3.5em}{
\centering
\tikz[scale = .6]{
\fill[fill=black!30] (-1,1) rectangle (1,-1);
\axes}
}
&
\parbox[c][3em][c]{3.5em}{
\centering
\tikz[scale = .6]{
\fill[fill=black!30] (-1,1) rectangle (1,0);
\axes}
}
&
\parbox[c][3em][c]{3.5em}{
\centering
\tikz[scale = .6]{
\fill[fill=black!30] (-1,1) rectangle (0,0);
\axes}
}
\\
\hline
\parbox[c]{5em}{$0\!<\!A_{i,n}\!<\!1$}
&
\parbox[c][5em][c]{3em}{
\centering
\tikz[scale = .6]{
\fill[fill=black!30] (-1,1) rectangle (0,-1);
\axes}
}
&
\parbox[c][3em][c]{3em}{
\centering
\tikz[scale = .6]{
\fill[fill=black!30] (-1,1) rectangle (1,-1);
\axes}
}
&
\parbox[c][3em][c]{3em}{
\centering
\tikz[scale = .6]{
\fill[fill=black!30] (-1,1) rectangle (0,-1);
\axes}
}
\\
\hline
\parbox[c]{4em}{$A_{i,n}\!=\!1$}
&
\parbox[c][5em][c]{3em}{
\centering
\tikz[scale = .6]{
\fill[fill=black!30] (-1,1) rectangle (0,0);
\axes}
}
&
\parbox[c][3em][c]{3em}{
\centering
\tikz[scale = .6]{
\fill[fill=black!30] (-1,1) rectangle (1,0);
\axes}
}
&
\parbox[c]{3.5em}{\centering not\\ aliased}
\\
\hline
\end{tabular}
%\caption{Constrains on $\feasreg_1$,
%$j\in\supp\col\setminus\{n\sm1,n\}$.
%% for columns $\{c,n\}$, outgoing into state $j\in\supp\col\backslash\{n,c\}$.
%}
\caption{
For any state
$i\in\X\setminus\{n\sm1,n\}$ with $A_{i,n}\in\{0,1\}$ or $A_{i,n\sm1}\in\{0,1\}$,
pick the relevant diagram.
}
\label{table:columns}
}
\end{table}
%\hfill
\begin{table}%{.3\textwidth}%
\centering
{%\scriptsize
\begin{tabular}{c|c}
\hline
$\a_{j}\!=\!0$
&
\parbox[c][5em][c]{3em}{
\centering
\tikz[scale = .6]{
\fill[fill=black!30] (-1,0) rectangle (1,-1);
\axes
}
}
%\\
%\hline
%$\a_{i}\in(0,1)$
%%%%%%%%%% Second row %%%%%%%%%%%%%%%%%%%%%%%%%%%5
%&
%\parbox[c][5em][c]{3em}{
%\centering
%\tikz[scale = .6]{
%\fill[fill=black!30] (-1,1) rectangle (1,-1);
%\draw (-1,0) -- (1,0);
%\draw (0, -1) -- (0, 1);
%\draw [fill] (0,0) circle   (1.5pt);
%}
%}
\\
\hline
$\a_{j}\!=\!1$
%%%%%%%%% Third row %%%%%%%%%%%%%%%%%%%%%%%%%%%5
&
\parbox[c][5em][c]{3em}{
\centering
\tikz[scale = .6]{
\fill[fill=black!30] (0,1) rectangle (1,-1);
\axes}
}
\\
\hline
\end{tabular}
\caption
%{Constrains on $\feasreg_1$, $i\in\supp\row\setminus\{n\sm1,n\}$.}
{For any state $j\in\supp\row\setminus\{n\sm1,n\}$ with $\a_j\in\{0,1\}$
	 pick the corresponding diagram.}
% $\{c,n\}$ for incoming from state $i\in\supp\row\backslash\{c,n\}$.}
\label{table:rows}
}
\end{table}

%\\[10pt]
\begin{table*}[t]%{\textwidth}
\centering
{%\scriptsize
\begin{tabular}{c|c|c|c|c}
&
$A_{{n\sm1},{n\sm1}}\!<\!A_{n,n}$
&
$A_{{n\sm1},{n\sm1}}\!=\!A_{n,n}$
&
\begin{tabular}{l}
$A_{n,n} < $ \\ $A_{{n\sm1},{n\sm1}}$ \\ $  < A_{\ot,{n\sm1}}$
\end{tabular}
&
$A_{{n\sm1},{n\sm1}}\!=\!A_{\ot,{n\sm1}}$
\\
%%%%%%%%% Second row %%%%%%%%%%%%%%%%%%%%%%%%%%%5
\hline
\parbox[c]{5em}{
\centering $A_{n\sm1,n} =0$}
&
\parbox[c][5em][c]{3.5em}{
\centering
\tikz[scale = .6]{
\fill[fill=black!30] (-1,0) rectangle (1,-1);
\axes}
}
&
\parbox[c][3em][c]{3.5em}{
\centering
\tikz[scale = .6]{
\fill[fill=black!30] (-1,.1) rectangle (1,-.1);
\axes}
}
&
\parbox[c][3em][c]{3.5em}{
\centering
\tikz[scale = .6]{
\fill[fill=black!30] (-1,0) rectangle (1,1);
\axes}
}
&
\parbox[c][3em][c]{3.5em}{
\centering
\tikz[scale = .6]{
\fill[fill=black!30] (0,1) rectangle (1,0);
\axes}
}
%%%%%%%%% Third row %%%%%%%%%%%%%%%%%%%%%%%%%%%5
\\
\hline
\parbox[c]{5em}{\centering $A_{n\sm1,n}> 0$}
&
\parbox[c][5em][c]{3em}{
\centering
\tikz[scale = .6]{
\fill[fill=black!30] (-1,1) rectangle (1,-1);
\axes}
}
&
\parbox[c][3em][c]{3em}{
\centering
\tikz[scale = .6]{
\fill[fill=black!30] (-1,1) rectangle (1,-1);
\axes}
}
&
\parbox[c][3em][c]{3em}{
\centering
\tikz[scale = .6]{
\fill[fill=black!30] (-1,1) rectangle (1,-1);
\axes}
}
&
\parbox[c][3em][c]{3em}{
\centering
\tikz[scale = .6]{
\fill[fill=black!30] (0,1) rectangle (1,-1);
\axes}
}
\\
\hline
\end{tabular}
\caption{
If $\a_{n\sm1} \geq \a_{n}$ pick the relevant diagram from here.
%Constrains on $\feasreg_2$ ($\a_{n\sm1} \geq \a_n$).
}
\label{table:feas2}
}
\end{table*}
%\\[10pt]
\begin{table}%{.8\textwidth}
\centering
{%\scriptsize
\begin{tabular}{c|c|c}
&
$A_{{n\sm1},{n\sm1}} \!=\! 0$
&
$A_{{n\sm1},{n\sm1}} \!>\! 0$
\\
%%%%%%%%% Second row %%%%%%%%%%%%%%%%%%%%%%%%%%%5
\hline
\parbox[c]{4em}{$A_{n,n}\!=\!0$}
&
\parbox[c][5em][c]{3.5em}{
\centering
\tikz[scale = .6]{
\fill[fill=black!30] (-1,0.9) -- (-1,1) -- (-.9, 1) -- (1, -.9) -- (1, -1) -- (.9,-1) -- (-1,0.9);
\axes}
}
&
\parbox[c][3em][c]{3.5em}{
\centering
\tikz[scale = .6]{
\fill[fill=black!30] (-1,1) -- (1,1) -- (1,-1) -- (-1,1);
\axes}
}
%%%%%%%%% Third row %%%%%%%%%%%%%%%%%%%%%%%%%%%5
\\
\hline
\parbox[c]{4em}{$A_{n,n} \!>\! 0$}
&
\parbox[c][5em][c]{3em}{
\centering
\tikz[scale = .6]{
\fill[fill=black!30] (-1,1) -- (1,-1) -- (-1,-1) -- (-1,1);
\axes}
}
&
\parbox[c][3em][c]{3em}{
\centering
\tikz[scale = .6]{
\fill[fill=black!30] (-1,1) rectangle (1,-1);
\axes}
}
\\
\hline
\end{tabular}
\caption{
If $\a_{n\sm1} < \a_{n}$ pick the relevant diagram from here.
%Constrains on $\feasreg_3$ ($\a_{n\sm1} < \a_n$).
}
\label{table:feas3}
}
\end{table}
}
%\caption{
%Determining the effective feasible region of a minimal stationary 2A-HMM.
%The shaded areas represent the constraints on $(\dt_{n\sm1},\dt_n)<\!<1$ ensuring $A_H(1 +\dt_{n\sm1},\dt_n)\geq 0$, as determined by the entries of the transition matrix $A$. The effective feasible region is obtained by taking the intersection of all the relevant regions (see Alg.\ref{alg:feasreg}).
%A 2A-HMM satisfying a set of constraints with their intersection resulting in a point like diagram (Fig.\ref{fig:tau-feas-examp-point}) has a unique solution.}
%\label{tab:constrains-all}
%\end{table*}
\subsection{Conditions for $\abs{\feasreg_H}=1$}
\begin{figure}%
\centering
~
%\begin{subfigure}%{.3\textwidth}
\scriptsize
%\centering
\tikz[scale = .6]{
\fill[fill=black!30] (-1,1) rectangle (1,-1);
\draw (-1,0) -> (1,0);
\draw (0, -1) -> (0, 1);
\draw [fill] (0,0) circle   (1.5pt);
\node at (1.7,0) {$\dt_{n\sm1}$};
\node at (0, 1.3) {$\dt_n$};
%\node at (-1.9,0) {$$};
\node at (0,-1.5) {\small no constraints};
}
%\includegraphics[width=\columnwidth]{filename}%
%\label{fig:tau-feas-examp-no-const}%
%\end{subfigure}
%
%\begin{subfigure}%{.3\textwidth}
\scriptsize
%\centering
\hspace{10pt}
\tikz[scale = .6]{
\centering
\fill[ fill=black!30] (-1,1) rectangle (1,0);
\draw (-1,0) -> (1,0);
\draw (0, -1) -> (0, 1);
\draw [fill] (0,0) circle   (1.5pt);
\node at (1.7,0) {$\dt_{n\sm1}$};
\node at (0, 1.3) {$\dt_n$};
%\node at (-1.9,0) {$$};
\node at (0,-1.5) {$\dt_n \geq 0$};
}
%\caption{$\dt_n \geq 0$}%
%\label{fig:tau-feas-examp-diag-const}%
%\end{subfigure}
%
%\begin{subfigure}%{.3\textwidth}
\scriptsize
%\centering
\hspace{5pt}
\tikz[scale = .6]{
\centering
%\fill[fill=black!20] (-1,1) rectangle (1,0);
\draw (-1,0) -> (1,0);
\draw (0, -1) -> (0, 1);
\draw [fill] (0,0) circle   (1.5pt);
\node at (1.7,0) {$\dt_{n\sm1}$};
\node at (0, 1.3) {$\dt_n$};
%\node at (-1.9,0) {$$};
\node at (0,-1.5) {$\dt_{n\sm1},\dt_n\!=\!0$};
}
%\includegraphics[width=\columnwidth]{filename}%
%\caption{$\dt_{n\sm1}\!=\!\dt_n\!=\!0$}%
%\label{fig:tau-feas-examp-point}%
%\end{subfigure}
\caption{\normalsize The effective feasible region for various constraints, ensuring $A_H(1+\dt_{n\sm1},\dt_n) \geq 0$ for $\abs{\dt_{n\sm1}},\abs{\dt_n}<\!< 1$.
%The shaded area is the feasible region for $\dt_{n\sm1},\dt_n <\!< 1$
}
\label{fig:tau-feas-examp-all}
\end{figure}
\label{app:cond_feasreg}
Let us first write
\beq
(\t_{n\sm1},\t_n) = (1,0) + (\dt_{n\sm1},\dt_n).
\eeq
We characterize the conditions for $\abs{\feasreg_H}=1$ by determining the geometrical constraints the entries of the transition matrix $A$ pose on $(\dt_{n\sm1},\dt_n)$ in order to ensure $A_H(1 +\dt_{n\sm1},\dt_n)\geq 0$. Note that $\abs{\feasreg_H}=1$ iff these constraints imply that $(\dt_{n\sm1},\dt_n)=(0,0)$ is the {\em unique} feasible pair.
%, and consequently $H$ will be identifiable.

As a first example, consider
a 2A-HMM $H$ having a transition matrix with all entries being strictly positive, $A \geq \eps > 0$.
Since the mapping (\ref{eq:Atrans}) is continuous in $\dt_{n\sm1},\dt_n$, 
there exists a neighborhood $N\subset\R^2$ of $(\dt_{n\sm1},\dt_n) = (0,0)$, such that for {\em any} 
$(\dt_{n\sm1},{\dt_n})\in N$ the matrix $A_H(1 +\dt_{n\sm1},\dt_n)$ is non-negative, and thus $N\subset\feasreg_{H}$. 
This condition can be represented in the $(\dt_{n\sm1},\dt_n)$ plane (i.e. $\R^2$) as the "full" diagram 
Fig.\ref{fig:tau-feas-examp-all}.
%\ref{fig:tau-feas-examp-no-const}.
On the other hand,
the condition that
$(\dt_{n\sm1},\dt_n)= (0,0)$ is the {unique} feasible pair
can be represented by a point like diagram as in 
Fig.\ref{fig:tau-feas-examp-all}.
%\ref{fig:tau-feas-examp-point}.

In general, 
the entries of the transition matrix $A$ put constraints on the feasible $(\dt_{n\sm1},\dt_n)$ 
%$\feasreg_H$ 
only when $(\t_{n\sm1},\t_n) = (1,0)$ is on the boundary of $\feasreg_H$.
These constraints can be explicitly determined in terms of $A$'s entries by considering
the exact characterization of $\feasreg_H$ given in Theorem \ref{lem:feasreg}.
%we are able to relate the entries of $A$ to the conditions for $(\t_{n\sm1},\t_n) = (1,0)$ to be on the boundary of $\feasreg_H$, and consequently determine the constraints posed on the feasible $(\dt_{n\sm1},\dt_n)$.
Note however that 
by the fact that $\feasreg_H$ is connected,
and as far as the condition $\abs{\feasreg_H}=1$ is concerned,
we  
only 
need to 
consider the shape of these constraints in a small neighborhood of $(\t_{n\sm1},\t_n) = (1,0)$,
i.e for $\abs{\dt_{n\sm1}},\abs{\dt_n} <\!< 1$.
Any such neighborhood can be represented on the $\R^2$ plane (as in Fig.\ref{fig:tau-feas-examp-all}).
The shape of this neighborhood for a given $H$ is called
the {\em effective feasible region} of $\feasreg_H$.

Now,
as the example with $A\geq\eps>0$ shows,
a non-trivial constraint on the (effective) feasible region must results from $A$ having some zeros entries.
Each such a zero entry, as determined by its position in $A$, put a boundary constraint on  $(\dt_{n\sm1},\dt_n)$.
These in turn corresponds to a suitable diagram in $\R^2$ 
(as the diagram for $\dt_n \geq 0$ in 
Fig.\ref{fig:tau-feas-examp-all}).
%\ref{fig:tau-feas-examp-diag-const}).
The effective feasible region of $A$ is obtained by taking the intersection of all these diagrams.
The exact correspondence between $A$'s entries and the corresponding diagrams is given in Tables \ref{table:columns},\ref{table:rows},\ref{table:feas2},\ref{table:feas3}.
% \ref{tab:constrains-all}.
The procedure for determining the effective feasible region of a 2A-HMM is given in Algorithm \ref{alg:feasreg}.
The correctness of the algorithm is demonstrated in the proof of
Lemma \ref{thm:feasreg}.

\begin{algorithm}
\caption{
determining the effective feasible region for minimal 2A-HMM $H$
}
\label{alg:feasreg}
\begin{algorithmic}[1]
%\FUNCTION{EffectiveFeasibleRegion}{$H$}
\STATE permute aliased states so that $P(\ot\gn n\sm1) \geq P(\ot \gn n)$
\STATE collect the following diagrams:
 \begin{itemize}
       \item[-] {$\forall i\in[n\sm2]$}
               with {$A_{i,n}\in\{0,1\}$ or $A_{i,n\sm1}\in\{0,1\}$} 
  pick the relevant diagram in Table \ref{table:columns} 
 \item[-]
 %$j\in\supp\row\setminus\{n\sm1,n\}$
 $\forall j\in\supp\row\setminus\{n\sm1,n\}$
  with $\a_j\in\{0,1\}$
        pick corresponding diagram in Table \ref{table:rows}
  \item[-] if $\a_{n\sm1} \geq \a_{n}$ pick relevant diagram in table \ref{table:feas2}
       and if $\a_{n\sm1} < \a_{n}$ pick relevant diagram in Table \ref{table:feas3}
\end{itemize}
\STATE {\bf Return} the intersection of all the regions obtained in previous step
%\ENDFUNCTION
\end{algorithmic}
\hide{}
\end{algorithm}

\subsection{Examples.}
\label{sec:examples}
%
%%\begin{figure}%
%%\centering
%%\begin{subfigure}%{.45\textwidth}
%%\centering
%%\input{ratchet.tex}
%%\caption{A 2-aliased ratchet.}%
%%%\label{fig:ratchet}%
%%\end{subfigure}
%%\begin{subfigure}%{.45\textwidth}
%%\centering
%%\input{exampleFeas2.tex}
%%\caption{A 2-aliased stochastic loop.}%
%%\label{fig:feas2}%
%%\end{subfigure}
%%\caption{Examples of 2A-HMMs having a unique solution. Aliased states are indicated by double circles.}%
%%\label{}%
%%\end{figure}
%
We demonstrate our Algorithm \ref{alg:feasreg} for determining the identifiability of 2A-HMMs on the 2A-HMM given in Section \ref{sec:simul}, shown in Fig \ref{fig:hmm_ion} (left).
%In this example, while states $1$ and $2$ are fully observable by emitting the ouputs $\bar 1$ and $\bar 2$ respectively, the aliased states $3,4$ emits the same output, say $\bar{3}$. The transition matrix is given by
%\beq
%A \!&\!=\!\left(\begin{array}{cc|cc}
  %p_1 & 0 & {\color{red}0} & 1\sm p_4 \\
  %0 & p_2 & 1\sm p_3 & {\color{blue}0} \\[2pt]
  %\hline
  %1\sm p_1 & {\color{purple}0} & p_3 & 0 \\
  %{\color{green}0} & 1\sm p_2 & 0 & p_4
%\end{array}
%\right),
%\eeq
%for some $p_1,p_2,p_3,p_4\in(0,1)$. 
Going through the steps of Algorithm \ref{alg:feasreg} 
we get the following diagrams for the effective feasible region:
\beq
%\label{diag:example1}
&\underbrace{
\parbox[c][3em][c]{3.5em}{
\centering
\tikz[scale = .5]{
\fill[fill=black!30] (-1,1) rectangle (0,-1);
\axes}
}
}_{
\substack{
\text{from Table \ref{table:columns}:}\\[3pt]
%\color{red} 
A_{1,3} = 0 \,\wedge\, A_{1,4} \neq 0}
}
\cap
\underbrace{
\parbox[c][3em][c]{3.5em}{
\centering
\tikz[scale = .5]{
\fill[fill=black!30] (-1,0) rectangle (1,1);
\axes}
}
}_{\substack{
\text{from Table \ref{table:columns}:}\\[3pt]
%\color{blue} 
A_{2,4} = 0 \,\wedge\, A_{2,3} \neq 0
}}
%\quad
\cap
\underbrace{
\parbox[c][3em][c]{3.5em}{
\centering
\tikz[scale = .5]{
\fill[fill=black!30] (0,1) rectangle (1,-1);
\axes}
}
}_{
\substack{
\text{from Table \ref{table:rows}:}\\[3pt]
%\color{green} 
\a_1 = 1}
}
\,
\cap
\,
\underbrace{
\parbox[c][3em][c]{3.5em}{
\centering
\tikz[scale = .5]{
\fill[fill=black!30] (-1,0) rectangle (1,-1);
\axes}
}
}_{
\substack{
\text{from Table \ref{table:rows}:}\\[3pt]
%\color{purple} 
\a_2 = 0 }
}
%\\
%&\hspace{50pt}
=
\,
\parbox[c][3em][c]{3.5em}{
\centering
\tikz[scale = .5]{
\axes}
}.
\eeq
Since their intersection results in a point like diagram, this 2A-HMM is identifiable.

More generally, for a minimal stationary 2A-HMM satisfying Assumptions (A1-A2) with aliased states $n$ and $n\sm1$,
a sufficient condition for uniqueness
is the following constraints on the allowed transitions between the hidden states:
$\exists i_{n\sm1},j_{n\sm1},i_n,j_n\in [n\sm2]$ such that
\begin{align*}
\checkmark\quad  & i_{n\sm1} \to {n\sm1} \, \to j_{n\sm1}  & \xmark\quad & i_{n\sm1} \to n \to \,*  \\
\checkmark\quad  & i_n \to n \to j_n  & \xmark\quad & i_n \to {n\sm1}\, \to \,*  \\
                                       &                                                & \xmark\quad & *\; \to {n\sm1}\, \to j_n  \\
                                               &                                                        & \xmark\quad & *\; \to n \to j_{n\sm1}.
\end{align*}
One can check that these conditions give the same set of diagrams as above.
%
%As a second example
%consider the 2A-HMM given in 
%Fig. \ref{fig:feas2}.
%The resulting region as determined by Algorithm \ref{alg:feasreg} gives
%\beq
%\underbrace{
%\parbox[c][3em][c]{3.5em}{
%\centering
%\tikz[scale = .5]{
%\fill[fill=black!30] (-1,1) rectangle (0,-1);
%\axes}
%}
%}_{
%\substack{
%\text{from Table \ref{table:columns}:}\\[3pt]
%A_{1,3} = 0 \,\wedge\, A_{1,4} \neq 0}
%}
%%\quad
%\cap
%\underbrace{
%\parbox[c][3em][c]{3.5em}{
%\centering
%\tikz[scale = .5]{
%\fill[fill=black!30] (0,1) rectangle (1,-1);
%\axes}
%}
%}_{
%\substack{
%\text{from Table \ref{table:rows}:}\\[3pt]\a_1 = 1}
%}
%\cap
%\underbrace{
%\parbox[c][3em][c]{3.5em}{
%\centering
%\tikz[scale = .5]{
%\fill[fill=black!30] (-1,.1) rectangle (1,-.1);
%\axes}
%}
%}_{
%\substack{
%\text{from Table \ref{table:feas2}:}\\[3pt]
%A_{3,4} = 0 \,\wedge\, A_{3,3} = A_{4,4}}
%}
%=
%\quad
%\parbox[c][3em][c]{3.5em}{
%\centering
%\tikz[scale = .5]{
%\axes}
%}.
%\eeq
%Thus, this 2A-HMM also has a unique solution.

\section{Proofs for Section \ref{sec:learn} (Learning)}
 
\subsection{Proof of Lemma \ref{lem:del-rel}}  

%\begin{proof}
The claim in (\ref{eq:momtoA1}) that $\mom^{(1)} = BAC_\b =  \At$ follows directly from (\ref{eq:mom-2A}) and (\ref{eq:striped-mom}).
Next, we have
\beq
\dmom{2} &=& BAAC_\b - BAC_\b BAC_\b
\\
&=& BA(I_n - C_\b B)AC_\b
\\
& = & BA \uo \vo_\b^{\tran} AC_\b,
\eeq
where the last equality is by the fact that $(I_n - C_\b B)=\uo \vo_\b^{\tran}$.
Since 
$BA \uo = \dav$ and $\vo_\b^{\tran} AC_\b = (\dalpv)^{\tran}$ we have $\dmom{2}=\dav (\dalpv)^{\tran}$ as claimed in (\ref{eq:dmomtoA2}).

As for (\ref{eq:dmomtoA3}) we have,
\beq
\dmom{3} &=& BAAAC_\b - BAAC_\b BAC_\b
\\
&& -\, BAC_\b BAAC_\b + BAC_\b BAC_\b BAC_\b
\\
& = & BA(I_n - C_\b B) A (I_n - C_\b B) AC_\b
\\
&=& \dav \vo_\b^{\tran} A \uo (\dalpv)^{\tran}.
\eeq
Since by definition $\dmomo = \vo_\b^{\tran} A \uo$
we have $\dmom{3} = \dmomo \dmom{2}$
and the claim in (\ref{eq:dmomtoA3}) is proved.

Finally,
\beq
\dmomT{c}
= BA \Big(\diag(\kernelExp_{[\cdot,c]}) - C_\b \diag(\kernel_{[\cdot,c]}) B  \Big)  AC_\b.
\eeq
Since
\beq
 \diag(\kernelExp_{[\cdot,c]}) - C_\b \diag(\kernel_{[\cdot,c]})B
 = \kernel_{n\sm1,c} \uo (\vo_\b)^{\tran}
\eeq
we have that $\dmomT{c} = \kernel_{n\sm1,c} \dmom{2}$ as claimed in (\ref{eq:dmomTtoA}).
%\end{proof} 
 
\subsection{Proof of Lemma \ref{lem:err-asymp}} 
%\begin{proof}%[Proof of Lemma \ref{lem:err-asymp}]
Assumption (A2) combined with the 
fact that the HMM has a finite number of states imply that the HMM is {\em geometrically} ergodic:
 there exist parameters $G < \infty$ and $\mix\in[0, 1)$ such that from any initial distribution $\initv$,
\beqn
\label{eq:mix}
\nrm{A^t \initv - \statv}_{1}
\leq
2G\mix^t
,\quad
\forall t \in \N.
\eeqn
Thus, we may apply the following concentration bound, given in \citet{kontweiss2014uniform}:
\begin{theorem}
\label{lem:hmm-conc}
Let $Y=Y_0,\ldots,Y_{T-1}\in\Y^T$ be the output of a HMM with transition matrix $A$ and output parameters $\prmv$. Assume that $A$ is geometrically ergodic with constants $G,\mix$.
Let $F:(Y_0,\ldots,Y_{T-1})\mapsto \R$
be any function that is Lipschitz wit constatnt $l$ with respect to the Hamming metric
on $\Y^T$.
Then, for all
$\eps>0$,
\beqn
\label{eq:kr}
\Pr(|F(Y)-\E F|>\eps T) \le 2 \exp\paren{-\frac{T(1-\mix)^2\eps^2}{2l^2G^2}}.
\eeqn
\end{theorem}
In order to apply the theorem note that $\forall t\in\{1,2,3\}$, $\E[\emomK_{ij}^{(t)}]= \momK_{ij}^{(t)}$ %and $\forall c\in[n\sm1]$, $\E[\emomTK_{ij}^{(c)}] = \momTK_{ij}^{(c)}$, 
for any $i,j\in[n\sm1]$.
In addition, following Assumption (A3),
$(T-t) \emomK_{ij}^{(t)}$
is $(t+1)L^2$-Lipschitz with respect to the Hamming metric on $\Y^T$.
Thus, taking $\eps\approx T^{-\oo{2}}$ in Theorem \ref{lem:hmm-conc} and applying a union bound on $i,j$ readily gives
\beq
\emomK^{(t)} & = & \momK^{(t)} + \O\paren{T^{-\oo{2}}}.
\eeq
The kernel-free moments $\emom^{(t)}$ given in (\ref{eq:smom}) incur additional error which results in a factor of at most
%entry of $\emom^{(t)}$ is $tL^t/
$1/(\sigma_{\min}(\kernel)^2 \min_i{\pi_i})$ hidden in the $O_P$ notation.
Since $\dmom{t}$ are (low order) polynomials of $\mom^{(t)}$, the asymptotics $\O\paren{T^{-\oo{2}}}$ carry on to the error in $\edmom{t}$.
A similar argument yields the claim for $\dmomT{c}$.
%\end{proof} 

\subsection{Proof of Lemma \ref{lem:detect}}
\label{app:detect}
%\begin{proof}
Let $\sigma_1$ and $\hat\sigma_1$ be the largest singular values of $\dmom{2}$ and $\edmom{2}$, respectively.
Combining Weyl's Theorem
\citep{stewart1990matrix}
with Lemma \ref{lem:err-asymp}
gives
\beq
\abs{\sigma_1 - \hat\sigma_1} \leq \nrm{\dmom{2} - \edmom{2}}_{\frob}
 =  O_P(T^{-\oo{2}}),
\eeq

Recall that under the null hypothesis $\H_0$, we have $\sigma_1=0$. Thus, with high probability $\hat\sigma_1 < \xi_0 T^{-\oo{2}}$, for some $\xi_0>0$.
In contrast, under $\H_1$ we have $\sigma_1=\sigma >0$, thus for some $\xi_1>0$, $\hat\sigma_1 > \sigma - \xi_1 T^{-\oo{2}}$.
Hence, 
taking $T$ sufficiently large, we have that for any $c_h>0$ and $0<\eps<\oo{2}$, with $h_T = c_h T^{-\oo{2} + \eps}$,
\beq
\text{in case } \H_0:&& \quad \hat\sigma_1 < h_T
\\
\text{in case } \H_1:&& \quad \hat\sigma_1 > h_T,
\eeq
with high probability.
Thus, the correct detection of aliasing is with high probability.
%\end{proof}

\newcommand{\score}{\operatorname{score}}
\subsection{Proof of Lemma \ref{lem:identify}}
%\begin{proof}
Let us define the following score function for any $i\in[n\sm1]$,
\beq
\score(i) = 
%\ac = \argmin_{i\in[n\sm1]} 
\sum_{j\in [n\sm1]} 
\nrm{\edmomT{j} - \kernel_{i,j} \edmom{2} }_{\frob}^2.
\eeq
According to Eq. (\ref{eq:acestimate}) the chosen aliased component is the index with minimal score.
Hence, in order to prove the Lemma we need to show that
\beq
\lim_{T\to\infty}       
\Pr(\exists i\neq n\sm1 : \score(i) < \score(n\sm1)) = 0.
\eeq    
By Lemma \ref{lem:err-asymp} and (\ref{eq:deltaGk}) we have
\beq
\edmomT{j} &=& \dmomT{j} + \frac{\xi_{F}^{(j)}}{\sqrt{T}}
\,=\, \kernel_{n\sm1,j}\dmom{2} + \frac{\xi_{F}^{(j)}}{\sqrt{T}}
\\
\edmom{2} &=& \dmom{2} + \frac{\xi_R}{\sqrt{T}},
\eeq
for some 
$\xi_R,\xi_{F}^{(j)}\in\R^{(n\sm1)\times (n\sm1)}$ with $O_P(\xi_R)=1$ and $O_P(\xi_F^{(j)})=1$.
Thus,
\beq
\score(n\sm1) = 
\oo{\sqrt{T}}\sum_{j\in [n\sm1]} 
\nrm{\xi_{F}^{(j)} - \kernel_{j,n\sm1} \xi_R }_{\frob}^2
%\\
%&&
%\displaystyle
\xrightarrow{P} 0.
\eeq

In contrast,
for any $i\neq n\sm 1$ we may write $\score(i)$ as
\beq
\sum_{j\in [n\sm1]} 
\nrm{(\kernel_{j,i} - \kernel_{j,n\sm1}) \dmom{2} + \oo{\sqrt{T}}( \xi_{F}^{(j)} - \kernel_{j,n\sm1} \xi_R) }_{\frob}^2.
\eeq
Applying the (inverse) triangle inequality we have
\beq
\score(i) \geq \sigma^2 \nrm{\kernel_{[\cdot,n\sm1]} - \kernel_{[\cdot,i]}}^2 - O_P(T^{-\oo{2}}).
\eeq
Since $\kernel$ is full rank, 
$\sigma^2\nrm{\kernel_{[\cdot,n\sm1]} - \kernel_{[\cdot,i]}}^2 > 0 $.
Thus, for any $i\neq n\sm1$ as $T\to\infty$, w.h.p $\score(i)> \score(n\sm1)$.
Taking a union bound over $i$ yields the claim.
%\end{proof}

\subsection{Estimating $\g$ and $\b$} 
\label{app:estimate_gb}
\newcommand{\obj}{h}
We now show how to 
estimate $\g$ and $\b$. 
%(\ref{eq:decomp-gb}) by
As discussed in Section \ref{sec:learnaliasA},
this is done by searching for $\g',\b'$ ensuring 
the non-negativity of (\ref{eq:decomp-gb}),
namely,
$A'_H(\g',\b')\geq 0$,
where
\beq
A'_H(\g',\b') \equiv C_{\b'} \At B
+
\g' C_{\b'} \bm u
{\vo_{\b'}}^\tran
+
\frac{\sigma}{\g'} \uo 
\bm v^{\tran} B
+
\dmomo\,  \uo {\vo_{\b'}}^\tran.
\eeq
We pose this as a non-linear two dimensional optimization problem.
For any $\g'\geq0$ and $0\leq\b'\leq1$ define the objective function $\obj:\R^2\to\R$ by
\beq
\obj(\g',\b') = \min_{i,j\in[n]} \{\g' A'_H(\g',\b')_{ij} \}.
\eeq
Note that $\obj(\g',\b')\geq 0$ iff $A'_H(\g',\b')$ does not have negative entries.
Recall that by the identifiability of $H$, if we constrain $\g'\geq0$ then the constraint $A'_H(\g',\b') \geq 0$ has the unique solution $(\g,\b)$ (this is the equivalent to the convention $\t_{n\sm1}\geq \t_n$ made in Section \ref{sec:ident}).
Namely, any $(\g',\b')\neq (\g,\b)$ results in at least one negative entry in $A_H'(\g',\b')$.
Hence, $\obj(\g',\b')$ has a unique maximum, obtained at the true $(\g,\b)$.
In addition, since $\nrm{\bm u}_2=\nrm{\bm v}_2=1$, a feasible solution must have $\g' \leq 2/\sigma$.
%Finally note that multiplying $h(\g',\b')$ by $\g'$ leaves the solution invariant. 
%Moreover, $\g' h(\g',\b')$ is the minimal entry in a matrix with second order polynomials in its entries.
So our optimization problem is:
\beqn
\label{eq:opt}
(\hat \g,\hat \b) = \argmax_{(\g',\b')\in[0,\frac{2}{\sigma}]\times[0,1]} {h(\g',\b')}
%\\
%\nonumber
%\text{subject to:}
%&&
%0\leq \g' \leq 2/\sigma,
%\quad
%0 \leq \b' \leq 1.
\eeqn
This two dimensional optimization problem can be solved by either brute force or any non-linear problem solver.

In practice, we solve the optimization problem (\ref{eq:opt})
with the empirical estimates plugged in, 
that is
% its corresponding empirical version
\beq
\hat A'_H(\g',\b') = 
C_{\hat\b} \hat\At B
+
%{C_\b} (\dav) 
\g' C_{\b'} \hat{\bm u}_1
%\Delta{\bar{\bm  a}} 
\vo_{\b'}^\tran
+
\frac{\hat\sigma_1}{\g'} \uo 
%(\dalpv^\tran) {B}
\hat{\bm v}_1^{\tran} B
%\Delta{\bar {\bm \a}}^{\tran}
+
\hat\dmomo\,  \uo \vo_{\b'}^\tran.
\eeq
The empirical objective function $\hat\obj(\g',\b')$ is defined similarly. 
%
%In practice, we solve the optimization problem with $\hat h$ replacing $h$ (and $\hat\sigma_1$ replacing $\sigma$).
%
Such a perturbation may results
in 
a problem with many feasible solutions,
or worse, with no feasible solutions at all.
Nevertheless, as shown in the proof of Theorem \ref{thm:Ahat},
this method is consistent.
Namely, as $T\to\infty$, the above method will return an arbitrarily close solution (in $\nrm{\cdot}_{\frob}$)
to the true transition matrix $A$, with high probability.

\subsection{Proof of Theorem \ref{thm:Ahat}} 
\newcommand{\gh}{\hat \g}
\newcommand{\bh}{\hat \b}
\newcommand{\AHh}{\hat A_H}
\newcommand{\AH}{A_H}
\newcommand{\dl}{\delta}
\newcommand{\toP}{\xrightarrow{P}}
Recall the definitions of $\AH'(\g',\b')$ and its empirical version $\AHh'(\g',\b')$, given in the previous Section \ref{app:estimate_gb}.
To prove the theorem 
we show that
\beq
\nrm{\AHh'(\gh,\bh) - \AH'(\g,\b)}_{\frob} \xrightarrow{P} 0.
\eeq
Toward this goal we bound the l.h.s by
\beqn
\label{eq:two-terms}
\nrm{\AHh'(\gh,\bh) - \AH'(\gh,\bh)}_{\frob} 
+
\nrm{\AH'(\gh,\bh) - \AH'(\g,\b)}_{\frob},
\eeqn
and show that each term converges to $0$ in probability.

We shall need the following lemma, establishing the {\em pointwise} convergence in probability of $\AHh$ to $\AH$:
\begin{lemma}
\label{lem:pointwise}
%$\AHh$ converges pointwise to $\AH$ in probability:
For any $0<\g'$ and $0 \leq \b' \leq 1$,
\beq
\nrm{\AHh(\g',\b') - \AH(\g',\b')}_{\frob} 
= o_P(1).
\eeq 
\end{lemma}
\begin{proof}
By (\ref{eq:Abar-cons}), $\hat\At \toP \At$.
In addition, in Section
\ref{app:detect}
we saw
$\hat \sigma_1 \toP \sigma$ and 
one can easily show that 
$\edmomo\toP \dmomo$.
Thus, in order to prove the claim it suffices to show 
%So we are left to show the convergence 
that
$\hat{\bm u}_1 \toP \bm u$ and $\hat{\bm v}_1 \toP \bm v$.
By Wedin's Theorem  
\citep{stewart1990matrix}:
\beq 
\nrm{\hat{\bm u}_1 - \bm u}_2 \leq C\frac{\nrm{\edmom{2}-\dmom{2}}_2}{\sigma},
\eeq
for some $C>0$.
Combining this with Lemma \ref{lem:err-asymp} gives that $\nrm{\hat{\bm u}_1 - \bm u}_2= O_P(T^{-\oo{2}})$.
The same argument goes for $\nrm{\hat{\bm v}_1 - \bm v}_2$.
\end{proof}

We begin with the second term in (\ref{eq:two-terms}).
The first step is showing that the estimated parameters $\hat\g,\hat\b$ in (\ref{eq:opt}) converge with probability to the true parameters $\g,\b$.
We first need to following lemma, 
establishing the convergence of $\hat\obj$ to $\obj$ uniformly in probability:
%by appl\citet[Corollary 2.2]{newey1991uniform}.
%
\begin{lemma}
\label{lem:uniform-h}
For any $\eps>0$,
\beq
\Pr\paren{\sup_{(\g',\b')\in[0,2]\times[0,1]}\abs{\hat\obj(\g',\b') - \obj(\g',\b')} > \eps}
= o(1).
% \xrightarrow{P} 0.
\eeq
\end{lemma}
\begin{proof}
Note that $\hat\obj(\g',\b')$ is the value of the minimal entry of a matrix with all entries being  polynomials of $\g',\b'$ with bounded coefficients.
Thus $\hat\obj$ is Lipschitz.
In addition $[0,2]\times[0,1]$ is compact
and, similarly to Lemma \ref{lem:pointwise},
$\hat h(\g',\b')$ converges in probability pointwise to $ h(\g',\b')$.
Hence, the claim follows by \citet[Corollary 2.2]{newey1991uniform}.
\end{proof}

\begin{lemma} 
\label{lem:ghbh_to_gb}
$(\gh,\bh) \xrightarrow{P} (\g,\b)$. 
\end{lemma}

\begin{proof}
Recall that $(\gh,\bh)$ are the maximizers of $\hat\obj(\g',\b')$ and $(\g,\b)$ are the maximizers of $\obj(\g',\b')$, over $(\g',\b')\in[0,2]\times[0,1]$.
%Since
%$\sup\abs{\hat\obj - \obj} \leq_P \eps$
%we must have $\abs{\hat\obj(\gh,\bh) -\obj(\g,\b)} \leq_{P} \eps$.
%
To prove the claim we need to show that
for any $\dl>0$, 
\beq
\Pr\paren{\nrm{(\gh,\bh)-(\g,\b)} > \dl } = o(1).
\eeq
Toward this end define
\beq
\eps(\dl) \equiv \obj(\g,\b) - \max_{\nrm{(\g',\b')-(\g,\b)} > \dl } \obj(\g',\b').
\eeq
Note that $\eps(\dl)>0$ since $h(\g',\b')$ has 
the {\em unique} maximum $(\g,\b)$.

Now,by Lemma \ref{lem:uniform-h}, 
we have that
\beqn
\label{eq:unif-g}
\Pr\paren{
\sup_{\g',\b'}\abs{\hat\obj(\g',\b') - \obj(\g',\b')}
> \eps(\dl)/4} = o(1).
\eeqn
Thus, if we show that 
%(\ref{eq:unif-g})
$\sup\abs{\hat\obj - \obj} \leq \eps(\dl)/4$ 
implies $\nrm{(\gh,\bh)-(\g,\b)} \leq \dl$ then the claim is proved. 
So assume
\beq
\sup_{\g',\b'}\abs{\hat\obj(\g',\b') - \obj(\g',\b')}
\leq \eps(\dl)/4.
\eeq
Toward getting a contradiction let us assume that $\nrm{(\gh,\bh)-(\g,\b)} > \dl$.
Then the following relations hold:
\beq
\obj(\gh,\bh) &\leq& \obj(\g,\b) - \eps(\dl)
\\
\hat\obj(\gh,\bh) &\leq& \obj(\gh,\bh) + \eps(\dl)/4
\\
\hat\obj(\g,\b) &\geq& \obj(\g,\b) - \eps(\dl)/4.
\eeq
Thus,
\beq
\hat\obj(\gh,\bh) &\leq& \hat\obj(\g,\b) - \eps(\dl)/2,
\eeq
in contradiction to the optimality of $(\gh,\bh)$.
%To conclude,
%$$
%\nrm{(\gh,\bh)-(\g,\b)} \leq_P \dl.
%$$
\end{proof} 

By Lemma \ref{lem:ghbh_to_gb}, $(\gh,\bh) \xrightarrow{P} (\g,\b)$. 
Since $H$ is minimal, Theorem \ref{thm:2-alias-irr-cond} implies $\g>0$ and thus $\gh \geq_P \g/2$. 
In addition, $\AH$ is continuous in the compact set $[\g/2,2]\times[0,1]$. 
Thus, by the continuous mapping theorem we have
\beq
\nrm{\AH(\gh,\bh) - \AH(\g,\b)}_{\frob} 
\xrightarrow{P} 0.
\eeq
This proves the case for the right term of (\ref{eq:two-terms}).

The convergence in probability of the left term of (\ref{eq:two-terms}) to zero is a direct consequence of the following uniform convergence lemma:
\begin{lemma}
\beq
\sup_{(\g',\b')\in[\frac{\g}{2},2]\times[0,1]} \nrm{\AHh(\g',\b') - \AH(\g',\b')}_{\frob} = o_P(1).
\eeq
\end{lemma}
\begin{proof}
Since $\g'\geq \g/2$ we have that for any $i,j\in[n]$, ${\AHh(\g',\b')}_{ij}$ is Lipschitz.
In addition, by Lemma \ref{lem:pointwise}, for any $(\g',\b')\in[\frac{\g}{2},2]\times[0,1]$, each entry $\AHh(\g',\b')_{ij}$ converge pointwise in probability to $\AH(\g',\b')_{ij}$.
Finally, $[\frac{\g}{2},2]\times[0,1]$ is compact.
Thus, the claim follows from \citet[Corollary 2.2]{newey1991uniform} with an application of a union bound over $i,j\in[n]$.
\end{proof}

\end{document}